\documentclass[journal,twoside,web]{ieeecolor}
\usepackage{generic}
\usepackage{cite}
\usepackage{amsmath,amssymb,amsfonts}
\usepackage{algorithmic}
\usepackage{graphicx}
\usepackage{algorithm,algorithmic}
\usepackage{hyperref}
\usepackage{textcomp}
\usepackage{subcaption}
\usepackage{booktabs}
\usepackage{mathrsfs}
\usepackage{multirow}
\usepackage{diagbox}
\usepackage[export]{adjustbox} 
\usepackage{makecell}
\usepackage{lettrine}

\newtheorem{Rem}{Remark}	
	
\newtheorem{Thm}{Theorem}	
\newtheorem{Cor}{Corollary}	
\def\BibTeX{{\rm B\kern-.05em{\sc i\kern-.025em b}\kern-.08em
		T\kern-.1667em\lower.7ex\hbox{E}\kern-.125emX}}
\begin{document}
	\title{Rapid Vibration Suppression and Trajectory Tracking of a Serial Manipulator with Multi-Flexible Links}
	\author{Chengyi Wang, Yilong Huang and Ji Wang
		\thanks{C. Wang, Y. Huang and J. Wang are with the School of Aerospace Engineering, Xiamen University, Xiamen 361102, P. R. China (e-mail: 23220231151773@stu.xmu.edu.cn, huangyilong@stu.xmu.edu.cn, jiwang@xmu.edu.cn). }}
	\maketitle
	
	\begin{abstract}
		Flexible robotic manipulators (FRMs) offer advantages in lightweight design and large workspace, but their structural flexibility induces vibrations, accelerates fatigue, degrades tracking performance, and limits operational speed. These challenges are further amplified in multi-link serial manipulators, where increased overall length leads to greater structural flexibility. This article presents a backstepping output-feedback framework for fast vibration suppression and tip tracking of an n-degree-of-freedom serial flexible manipulator robot (nDSFMR). Each link-joint is modeled as a Timoshenko beam coupled with an ODE and transformed into a canonical hyperbolic PDE with boundary dynamics. A backstepping-based boundary controller at the joint is developed to equivalently inject distributed damping along the beam, enabling rapid vibration suppression and trajectory tracking, only using available boundary measurements. Experiments on a two-link flexible manipulator demonstrate faster vibration suppression and convergence of the end-effector to the desired trajectory, compared with a linear quadratic regulator (LQR) with feedforward control.
	\end{abstract}
	
	
	\section{Introduction}
	\label{sec:introduction}
	
	\subsection{Background and motivation}
	
	Robotic manipulators are ubiquitous in industrial applications.  Research in this field is generally categorized into two domains: rigid manipulators and flexible robotic manipulators(FRMs).
	Current research increasingly favors FRMs due to their distinctive advantages, including light weight, lower energy consumption, and high power-to-weight ratios \cite{bib12}, \cite{bib14}.
	
	The demand for lightweight manipulators has surged across diverse sectors, including the maintenance of explosive workspaces in satellite deployment and space station operations \cite{bib16}, high-precision semiconductor production \cite{bib17}, and surgical procedures where compact size and cost-effectiveness are crucial \cite{bib18}. Despite their advantages, these high-aspect-ratio lightweight manipulators suffer from pronounced structural flexibility, which not only degrades end-effector accuracy but also triggers persistent vibrations and accelerates structural fatigue. From a theoretical perspective, characterizing such systems is mathematically intensive, and incorporating elasticity yields complex governing equations with coupled spatio-temporal dynamics \cite{bib19}, making high-precision control of Flexible Robotic Manipulators (FRMs) particularly challenging under boundary actuation and sensing.
	
	\subsection{Mathematically Modeling}
	
	The dynamic behavior of FRMs has been extensively investigated over the past few decades \cite{bib15}.
	The governing equations of FRMs are approximated as finite-dimensional models by using some discretization techniques, such as the assumed mode method (AMM) \cite{bib25}, the finite element method (FEM) \cite{bib26}, or the lumped-parameter method \cite{bib24}. A novel approach was proposed in \cite{bib60} for linearization of equations of motion of FRMs leveraging screw theory and dual linear algebra for the derivations.
	
	To accurately represent complex flexibility and spatially continuous deformation, infinite-dimensional distributed parameter models are essential. The dynamics of such systems are typically derived via Hamilton's principle, which provides a systematic variational framework \cite{bib27}, \cite{bib28}. For instance, models of one-link and two-link flexible manipulators were derived using Hamilton's principle in \cite{bib29}, \cite{bib30}, respectively. Furthermore, the work in \cite{bib59} employs port-Hamiltonian formalism to model complex mechanical systems comprising both rigid bodies and flexible links.
	This principle has been further applied to establish detailed Timoshenko beam models for robotic arms, including the well-known Canadarm2, deployed on the International Space Station \cite{bib31}. Other effective modeling frameworks include Lagrangian equations \cite{bib32}, Newton-Euler equations \cite{bib33}, and Kane's method \cite{bib34}, are also good choices for modeling.

	\subsection{Control Strategies}
	Various control strategies have been developed for FRM models of FRMs \cite{bib23}, \cite{bib68}. For example, an LQR-based motion and vibration synthesized controller was proposed in \cite{bib37} for a multilink FRM under uncertain disturbances with arbitrary frequencies and unknown magnitudes. Additionally, model-based sliding mode control (SMC) was developed in \cite{bib38}, \cite{bib58} to generate driving torques. \cite{bib13} proposed a modified tip-based control strategy for flexible telescoping manipulators by leveraging a discretized Euler-Bernoulli beam framework.
	Distinct from these strategies, boundary control directly interacts with the boundaries of the FRM, offering significant advantages in mitigating wave propagation and enhancing control stability, particularly suited for suppressing vibrations at the extremities. The aforementioned control designs rely on finite-dimensional ODE approximations of the flexible arm, which inevitably lead to spillover effects.
	
	To enhance control accuracy and avoid spillover effects, a growing body of work has focused on controller design based on infinite-dimensional PDE models, like Euler-Bernoulli beams \cite{bib65} or Timoshenko beams \cite{bib66}. Incorporating the result in \cite{bib63}, the work in \cite{bib61} modeled the flexible robot fish as an Euler-Bernoulli beam to derive the control law. Considering boundary output constraints,
	a barrier controller was proposed to suppress vibrations in a Euler-Bernoulli beam modeled flexible manipulator in \cite{bib40}, in which disturbance rejection is incorporated in boundary control design, and \cite{bib39} introduced adaptive fault-tolerant control to mitigate the effects of actuator failures.
	Besides, robust output tracking for a Timoshenko beam using observer-based error feedback was proposed in \cite{bib47}. The work in \cite{bib48} explored stabilizing a rotating Timoshenko beam through force applied at the free end and torque at the pivoted end, while \cite{bib62} demonstrated stabilization of a rotating beam system requiring only torque control at the pivoted end. Practical applications include the stabilization of a one-link flexible arm interacting in soft environments  \cite{bib17} and LQR control of flexible beam Quanser experiments \cite{bib9}.
	
	The above work on boundary control of flexible manipulators is equivalent to placing damping at the joint to dissipate the vibration energy of the overall system. In contrast, the backstepping boundary control method, a powerful tool for infinite-dimensional systems, enables the spatially distributed placement of damping along the flexible manipulator using only joint-level actuation, thereby achieving arbitrarily fast vibration suppression. It was first applied to beam equations in \cite{bib50}, \cite{bib1}. This method was later extended in \cite{bib56} to design an output feedback control strategy for a slender Timoshenko beam model, where actuation occurs at the beam's base and sensing at the tip. Some studies have focused on backstepping control of the shear beam and Euler-Bernoulli beam, as shown in \cite{bib52}, \cite{bib53}. The Riemann transformation provides an effective way to equivalently represent beam PDEs in a lower-order form that is more amenable to control design. By Riemann transformation, the shear beam can be rewritten as coupled first-order hyperbolic PDEs with nonlocal terms, whose boundary control is given in \cite{bib4}. Recently, the work in \cite{bib55} applied Riemann transformation  to reformulate the Timoshenko beam as a $(2+2) \times (2+2)$ first-order hyperbolic PDE system, and achieved rapid stabilization of a Timoshenko beam.
	
	In this paper, using Hamilton's principle, the model for an n-degree-of-freedom serial flexible manipulator robot (nDSFMR) is derived. The flexible links are modeled as slender Timoshenko beams coupled with ODEs for the rigid joints. We propose a novel output-feedback control strategy featuring a backstepping boundary controller and an observer constructed from joint angles and the base strain measurements. Finally, the effectiveness of the proposed approach is validated through trajectory tracking experiments conducted on a 2DSFMR.
	
	\subsection{Main Contribution}
	1) Unlike most existing control designs that rely on finite-dimensional ODE approximations and consequently suffer from spillover effects \cite{bib58}, \cite{bib37}, \cite{bib13}, this work establishes a physics-consistent infinite-dimensional model of a flexible manipulator. The control design is carried out directly at the PDE level, effectively eliminating spillover and enabling high-accuracy control of flexible robotic systems.
	
	2) Existing PDE-based approaches typically suppress vibrations by injecting damping at specific locations \cite{bib63}, \cite{bib61}. In contrast, the proposed backstepping-based control achieves an equivalent spatially distributed damping injected along the entire flexible structure using only boundary actuation. The damping profile can be tuned via control design parameters, leading to rapid vibration suppression and fast convergence of the end-effector to the desired trajectory.
	
	3) The approach is validated both theoretically and experimentally, enabling fast vibration suppression and end-effector tracking.

	subsection{Organization}
	The rest of the article is organized as follows. Section \ref{Modeling} presents the problem formulation. The proposed output-feedback control framework is detailed in Section \ref{Output}, followed by experimental validation on the 2DSFMR in Section \ref{Experiment}. Section \ref{Conclusion} provides the conclusion and future work.
	
	\subsection{Notation}
	\begin{itemize}
		\item  Let $U \subseteq \mathbb{R}^n$ be an open set and let $\Omega \subseteq \mathbb{R}$ be a set. By $C^0(U;\Omega)$, we denote the class of continuous mappings on $U$, which take values in $\Omega$. By $C^k(U;\Omega)$, where $k \geq 1$, we denote the class of continuous functions on $U$, which have continuous derivatives of order $k$ on $U$ and take values in $\Omega$. Also, by $L^{\infty}(U,\Omega)$, we denote the class of (Bochner) measurable mappings on $U$, whose $\Omega$-norms are essentially bounded.
		\item We use the notation $L^2(0,1)$ for the standard space of the equivalence class of square-integrable, measurable functions $f:[0,1]\rightarrow \mathbb{R}$, with $\| f \|^2:=\int_{0}^{1}f(x)^2dx<+\infty$ for $f\in L^2(0,1)$. For an integer $k \geq 1$, $H^k(0,1)$ denotes the Sobolev space of functions in $L^2(0,1)$ with all its weak derivatives up to order $k$ in $L^2(0,1)$, with $\| f\|^2_{H^k}=\int_{0}^{1}(f(x)^2+f^{(1)}(x)^2+\cdot\cdot\cdot+f^{(k)}(x)^2)dx$.
		\item The notation $|\cdot|$ denotes Euclidean norm. The notation $\dot{z}(t)$ denotes the time derivative of $z$. The notation $c^{(i)}(x)$ denote the i times derivatives of $c$.
		\item Suppose that $f(x)$, $F(x,y)$ are functions defined in domain $\mathcal{D}$ where $\mathcal{D}=\lbrace (x,y)\in \mathbb{R}^2 \vert 0 \leq y \leq x \leq 1 \rbrace$, then the norm $| \cdot |_{\infty}$ is defined by $| f(x) |_{\infty}=\sup\limits_{x\in\mathcal{D}}\lvert f(x) \lvert$ and the norm $\| \cdot\|_{\infty}$ is defined by $\| F(x,y)\|_{\infty}=\sup\limits_{x,y\in\mathcal{D}}|F(x,y)|$.
	\end{itemize}
	
	\begin{figure}
		\centerline{\includegraphics[width=\columnwidth]{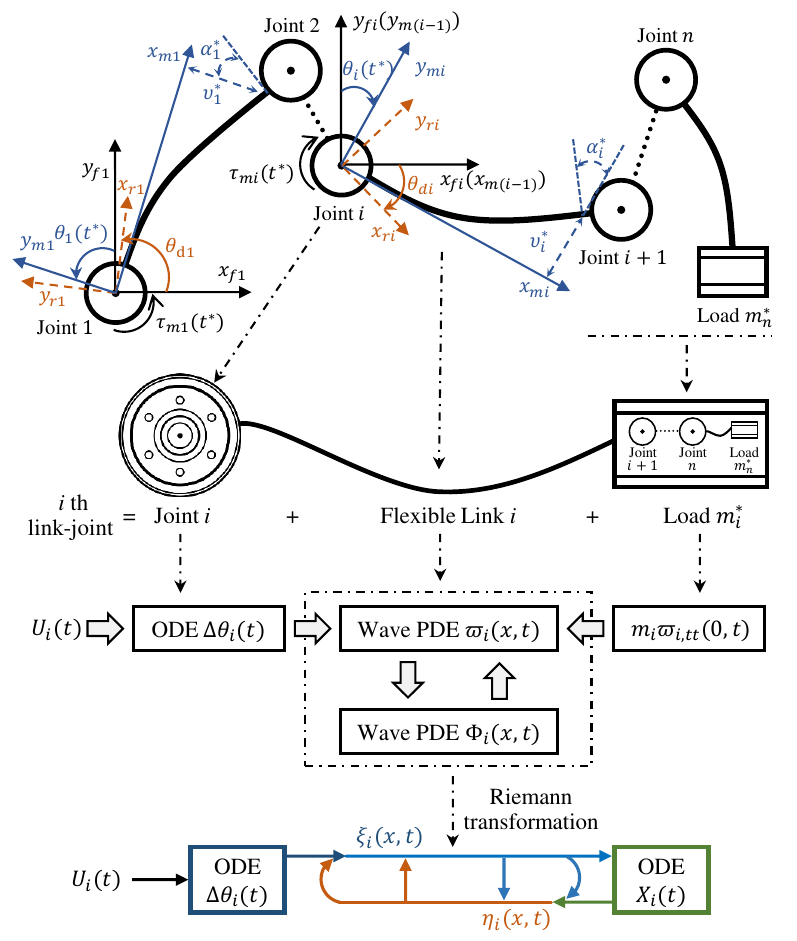}}
		\caption{n-degree-of-freedom serial flexible manipulator robot (nDSFMR) : from the physical model to the mathematical plant.}
		\label{Fig: multiple links}
	\end{figure}
	
	\section{Modeling of an nDSFMR}\label{Modeling}
	
	In this section, we derive the dynamic model of an n-degree-of-freedom serial flexible manipulator robot (nDSFMR) using Hamilton's principle.
	The nDSFMR is depicted in Fig. \ref{Fig: multiple links}, where $x_{fi}y_{fi}$ denotes the fixed coordinate frame of the $i$-th joint, $x_{mi}y_{mi}$ represents the moving frame rotating with the $i$-th joint. The $x_{ri}y_{ri}$ coordinate is a reference frame that rotates according to the reference trajectory of the $i$-th joint $\theta_{di}(t)$, which is obtained from the tip trajectory via a template-based FRM motion planner and satisfies
	\begin{align}
		\theta_{di}^{(j)}(t)\in L^\infty(0,\infty), i=1,...,n;j=1,2. \label{eq:conditons_theta_id}
	\end{align}
	
	This robot system is composed of serial link-joints where each flexible link is rigidly clamped to its joint, and also carries a harmonic joint at its distal end, to which the subsequent link is attached.
	From an applied perspective, the nDSFMR mechanism effectively captures the torsional compliance and serial linkage flexibility intrinsic to modern lightweight robotic manipulators. From a theoretical standpoint, the coupling between rigid-body motion and elastic deformation results in infinite-dimensional dynamics that necessitate active vibration suppression. Consequently, the nDSFMR platform serves as a vital benchmark for validating the control strategies designed for high-precision performance in the presence of spatio-temporal vibrations.
	
	\subsection{Modeling}
	
	The dynamic model of the nDSFMR is a serial concatenation of individual link subsystems, characterized by the moving frame of the $i$-th joint coinciding with the fixed frame of the $(i+1)$-th joint. For brevity, this section presents the model for the $i$-th link-joint of the nDSFMR Robot, derived via Hamilton's principle. The Lagrangian of the $i$-th link-joint $\acute L_i$, defined by kinetic energy-potential energy, is reached as follows
	
	\begin{align}
		\acute L_i=&KE_{i,trans}^\ast
		+KE_{i,rot}^\ast+\frac{1}{2}J_i^{\ast}\dot{\theta}_i(t^{\ast})^2-PE_{i,bending}^\ast \notag \\
		&-PE_{i,shear}^\ast \label{eq:L1}
	\end{align}
	where $KE_{i,trans}^\ast$, $KE_{i,rot}^\ast$, $\frac{1}{2}J_i^{\ast}\dot{\theta}_i(t^{\ast})^2$ are kinetic energy and $PE_{i,bending}^\ast$, $PE_{i,shear}^\ast$ are potential energy, both detailed in \cite{bib0}, and where $\theta_i(t^{\ast})$, $J_i^\ast$ are the angle of the moving reference $x_{mi}y_{mi}$ and the inertia moment of the $i$-th joint, respectively, with the virtual work
	
	\begin{align}
		\delta W_i=&(\tau_{mi}(t^\ast)+c_i^\ast \dot{\theta}_i(t^\ast))\delta(\theta_i(t^\ast))+m_i^\ast(\upsilon^\ast_{i,t^\ast t^\ast}(0,t^\ast)\notag \\
		&+(L_i^\ast+R_i^\ast)\ddot{\theta}_i(t^\ast))\delta(\upsilon_i^\ast(0,t^\ast)) \label{eq:virtual_work}
	\end{align}
	where $c_i^\ast$ is the damping coefficient of the $i$-th joint, $\tau_{mi}(t^\ast)$ is the torque of the $i$-th joint, $L_i^\ast$ is the length of the $i$-th flexible link, $\upsilon_i^\ast(x^\ast,t^\ast)$  is the transverse displacement of the $i$-th link in the moving frame $x_{mi}y_{mi}$, $R_i^\ast$ is the radius of the $i$-th joint and $m_i^\ast$ is the tip mass of the $i$-th link.
	
	Then, introducing the following dimensionless parameters
	\begin{align}
		&x=\frac{x^\ast}{L_i^\ast},\quad I_i=\frac{I_i^\ast}{{L_i^\ast}^4},\quad t=t^\ast \omega_0^\ast,\quad A_i=\frac{A_i^\ast}{{L_i^\ast}^2},\quad  \upsilon_i=\frac{\upsilon_i^\ast}{L_i^\ast}, \notag \\
		&R_i=\frac{R_i^\ast}{L_i^\ast},\quad G_i=\frac{G^\ast {L_i^\ast}^4}{E^\ast I_i^\ast},\quad \rho_i=\frac{\rho^\ast {L_i^\ast}^6 {\omega_0^\ast}^2}{E^\ast I_i^\ast},\quad J_i=J_i^\ast{\omega_0^\ast}^2 \notag \\
		&c_i=c_i^\ast \omega_0, \quad m_i=\frac{m_i^\ast L_i^\ast{\omega_0^\ast}^2}{E^\ast I_i^\ast} \label{parameters_dimensionless}
	\end{align}
	where $I_i^\ast$ is the area moment of inertia of the cross-section of the $i$-th link about the neutral axis $x_i^{\ast}$, $\omega_0^\ast$ is the natural frequency, $A_i^\ast$ is the cross sectional area of the $i$-th flexible link, $E^\ast$ is the modulus of elasticity, $G^\ast$ is the shear modulus, $\rho^\ast$ is the density of the beam, and then applying Hamilton's principle with the energy equations $\acute L_i$ \eqref{eq:L1} and virtual work $\delta W_i$ \eqref{eq:virtual_work}, the dynamic model of the $i$-th link-joint is derived as follows
	\begin{align}
		&\rho_i A_i((1+R_i-x)\ddot{\theta}_i(t)+\upsilon_{i,tt}(x,t)) \notag \\
		&+k'G_iA_i(\alpha_{i,x}(x,t)-\upsilon_{i,xx}(x,t))=0, \label{eq:upsilon} \\
		&\rho_i I_i(-\ddot{\theta}_i(t)+\alpha_{i,tt}(x,t))-\alpha_{i,xx}(x,t)\notag \\
		&+k'G_iA_i(\alpha_i(x,t)-\upsilon_{i,x}(x,t))=0,  \label{eq:alpha} \\
		&J_i\ddot{\theta}_i(t)=c_i\dot{\theta}_i(t)+\tau_{mi}(t) \label{eq:theta}
	\end{align}
	where $k'$ is the shear factor, $\alpha_i(x^\ast,t^\ast)$ is the angle of rotation of the cross-section of the $i$-th link due to the bending moment, with the clamped boundary conditions to be satisfied
	\begin{align}
		&\alpha_{i,x}(1,t)=0,\quad \alpha_{i,x}(0,t)=0, \quad \upsilon_{i}(1,t)=0, \label{eq:bc_alpha} \\
		& \upsilon_{i,x}(0,t)=m_i(\upsilon_{i,tt}(0,t)+(1+ R_i)\ddot{\theta}_i(t))+\alpha_i(0,t). \label{eq:bc_upsilon}
	\end{align}
	
	\subsection{Reformulation}
	
	To rewrite the system \eqref{eq:upsilon}–\eqref{eq:bc_upsilon} into a form suitable for control design, we introduce new variables to place the model \eqref{eq:upsilon}--\eqref{eq:bc_upsilon} in a reference frame $x_{ri}y_{ri}$ that aligns with the reference trajectory \eqref{eq:conditons_theta_id} we want to track
	\begin{align}
		\varpi_i(x,t)=&\upsilon_i(x,t)+(1+R_i-x)(\theta_i(t)-\theta_{di}(t)), \label{varpi} \\
		\Phi_i(x,t)=&\alpha_i(x,t)-(\theta_i(t)-\theta_{di}(t)), \label{Phi} \\
		\Delta \theta_i(t)=& \theta_i(t)-\theta_{di}(t), \label{det_theta1} \\
		\epsilon_i=&\frac{\rho_i}{k'G_i},\quad \mu_i=\rho_i I_i,\quad a_i=A_i\rho_i, \label{epsilon1}
	\end{align}
	and then the system \eqref{eq:upsilon}–\eqref{eq:bc_upsilon} is rewritten as
	
	\begin{align}
		\epsilon_i \varpi_{i,tt}(x,t)-\varpi_{i,xx}(x,t)=&-\epsilon_i(1+ R_i-x)\ddot{\theta}_{di}(t) \notag \\
		&-\Phi_{i,x}(x,t), \label{eq:varpi}\\
		\mu_i\Phi_{i,tt}(x,t)-\Phi_{i,xx}(x,t)=&-\frac{a_i}{\epsilon_i}(\Phi_i(x,t)-\varpi_{i,x}(x,t)) \notag \\
		&+\mu_i\ddot{\theta}_{di}(t), \label{eq:Phi}\\
		J_i\Delta\ddot{\theta}_i(t)=&c_i\Delta\theta_i(t)+U_i(t) \label{eq:det_theta1}
	\end{align}
	for $x \in (0,1)$, $t>0$,
	with boundary conditions
	\begin{align}
		\varpi_{i,t}(1,t)=&R_i\Delta\dot{\theta}_i(t), \quad \Phi_{i,x}(1,t)=0, \quad
		\Phi_{i,x}(0,t)=0, \label{eq:bc_varpi} \\
		\varpi_{i,x}(0,t)=&\Phi_i(0,t)+m_i(\varpi_{i,tt}(0,t)+(1+R_i)\ddot{\theta}_{di}(t)), \label{eq:bc_Phi}
	\end{align}
	where
	\begin{align}
		U_i(t)=\tau_{mi}(t)-J_i\ddot{\theta}_{di}(t)+c_i\theta_{di}(t).\label{eq:Ui1}
	\end{align}
	Please note that $U_i$ is to be designed in the following process, and then \eqref{eq:Ui1} will be used to obtain the actual control input $\tau_{mi}(t)$ for nDSFMR once $U_i(t)$ is built.
	
	Due to the flexible link’s thinness, we consider it a slender Timoshenko Beam, which implies that the parameter $\mu_i$ is very small. Following the approximation approach in \cite{bib1}, we set $\mu_i=0$ and $\Phi_i(0,t)=0$. Consequently, we get
	\begin{align}
		\frac{a_i}{\epsilon_i}(\Phi_i(x,t)-\varpi_{i,x}(x,t))-\Phi_{i,xx}(x,t)=0, \label{eq:Phi_mu0}
	\end{align}
	then using the Laplace transformation for \eqref{eq:Phi_mu0}, the solution of $\Phi_i$ is derived as
	\begin{align}
		\Phi_i(x,t)=-b_i\int_{0}^{x}\sinh(b_i(x-y))\varpi_{i,y}(y,t)dy \label{solution:Phi}
	\end{align}
	where $b_i=\sqrt{\frac{a_i}{\epsilon_i}}$. Differentiating $\Phi_i(x,t)$ with respect $x$ and substituting the result into \eqref{eq:varpi}, we obtain
	\begin{align}
		\epsilon_i\varpi_{i,tt}(x,t)=&\varpi_{i,xx}(x,t)-\epsilon_i(1+R_i-x)\ddot{\theta}_{di}(t) \notag \\
		&+b_i^2\int_{0}^{x}\cosh(b_i(x-y))\varpi_{i,y}(y,t)dy \label{eq:varpi_only}
	\end{align}
	with boundary conditions to be satisfied
	\begin{align}
		\varpi_{i,t}(1,t)=&R_i\Delta\dot{\theta}_i(t), \\ \varpi_{i,x}(0,t)=&m_i(\varpi_{i,tt}(0,t)+(1+R_i)\ddot{\theta}_{di}(t)). \label{eq:bc_varpi_only}
	\end{align}
	
	To facilitate the backstepping control design \cite{bib46}, \cite{bib2}, we use the following Riemann transformation
	\begin{align}
		\xi_i(x,t)&=\sqrt{\epsilon_i}\varpi_{i,t}(x,t)+\varpi_{i,x}(x,t), \label{xi1} \\
		\eta_i(x,t)&=\sqrt{\epsilon_i}\varpi_{i,t}(x,t)-\varpi_{i,x}(x,t), \label{eta1} \\
		\acute x_i(t)&=\varpi_i(0,t) \label{acute x1}
	\end{align}
	to convert the system \eqref{eq:varpi_only}--\eqref{eq:bc_varpi_only}  into a $2\times2$ coupled hyperbolic PDE-ODE system
	\begin{align}
		\sqrt{\epsilon_i}\xi_{i,t}(x,t)=&\xi_{i,x}(x,t)+\frac{b_i^2}{2}\int_{0}^{x}\cosh(b_i(x-y))(\xi_i(y,t) \notag \\
		&-\eta_i(y,t))dy-\epsilon_i(1+R_i-x)\ddot{\theta}_{di}(t), \label{eq:xi1} \\
		\sqrt{\epsilon_i}\eta_{i,t}(x,t)=&-\eta_{i,x}(x,t)-\epsilon_i(1+R_i-x)\ddot{\theta}_{di}(t)+\frac{b_i^2}{2} \notag \\
		&\int_{0}^{x}\cosh(b_i(x-y))(\xi_i(y,t)-\eta_i(y,t))dy,  \label{eq:eta1} \\
		\xi_i(1,t)=&-\eta_i(1,t)+2\sqrt{\epsilon_i}R_i\Delta\dot{\theta}_i(t),\label{eq:bc_xi1} \\
		\eta_i(0,t)=&-\xi_i(0,t)+C_iX_i(t), \label{eq:bc_eta1} \\
		\dot{X}_i(t)=&A_i\dot{X}_i(t)+B_i\xi_i(0,t)+D_i\ddot{\theta}_{di}(t), \label{eq:X1}
	\end{align}
	where $X_i(t)=[\dot{\acute x}_i(t),$ $\acute x_i(t)]'$, $A_i=[-\frac{\sqrt{\epsilon_i}}{m_i},$ $0;$ $1,$ $0]$, $B_i=[\frac{1}{m_i},$ $0]'$, $C_i=[2\sqrt{\epsilon_i},$ $0]$, $D_i=[-(1+R_i),$ $0]'$.
	
	The plant's initial conditions are taken as
	\begin{align}
		&	(\Delta \theta_i(0),\Delta \dot{\theta}_i(0),\xi_i(x,0),\eta_i(x,0), X_i(0)) \notag \\
		&	\in \mathcal H:=\mathbb R\times \mathbb R\times H^{1}(0,1)^2  \times \mathbb R^2 . \label{space:H}
	\end{align}
	
	Now, we obtain the model of the $i$-th link-joint \eqref{eq:det_theta1} and \eqref{eq:xi1}--\eqref{eq:X1} that is ready for control design.
	
	\section{Output-feedback control design of an nDSFMR}\label{Output}
	In this section, we present the state-feedback control design and observer design for the $i$-th link-joint to achieve the output-feedback control, which constitute an output-feedback control framework for an n-degree-of-freedom serial flexible manipulator robot (nDSFMR) based on the theoretical result established in \cite{bib0}.
	\subsection{State-Feedback Control Design}
	
	We propose a backstepping transformation to map the original system \eqref{eq:xi1}--\eqref{eq:X1} into a target system with the desired performance. Therefore, we built the backstepping transformation as
	\begin{align}
		\beta_i(x,t)=&\xi_i(x,t)+\gamma(x)X_i(t)-\int_{0}^{x}k(x,y)\xi_i(y,t)dy\notag \\
		&-\int_{0}^{x}l(x,y)\eta_i(y,t)dy \label{Trans:beta1}
	\end{align}
	whose inverse is
	\begin{align}
		\xi_i(x,t)=&\beta_i(x,t)-\lambda(x)X_i(t)+\int_{0}^{x}\rho(x,y)\beta_i(y,t)dy\notag \\	&+\int_{0}^{x}\sigma(x,y)\eta_(y,t)dy. \label{Trans:xi1}
	\end{align}
	
	The kernels $k(x,y)$, $l(x,y)$ evolving in the domain $\mathcal{D}=\lbrace (x,y)\in \mathbb{R}^2 | 0 \leq y \leq x \leq 1 \rbrace$ and the kernel $\gamma(x)$ defined in $\lbrace 0\leq x\leq 1 \rbrace$ satisfy
	\begin{align}
		k_x(x,y)=&-k_y(x,y)-\frac{b_i^2}{2}\cosh(b_i(x-y))\notag \\
		&+\frac{b_i^2}{2}\int_{y}^{x}\cosh(b_i(z-y))(l(x,z)+k(x,z))dz, \label{eq:k} \\
		l_x(x,y)=&l_y(x,y)+\frac{b_i^2}{2}\cosh(b_i(x-y))\notag \\
		&-\frac{b_i^2}{2}\int_{y}^{x}\cosh(b_i(z-y))(l(x,z)+k(x,z))dz, \label{eq:l} \\
		\gamma_x(x)&=\sqrt{\epsilon_i}\gamma(x)A_i-l(x,0)C_i, \label{eq:gamma} \\
		k(x,0)=&-l(x,0)-\sqrt{\epsilon_i}\gamma(x)B_i, \quad l(x,x)=0, \label{eq:bc_k_l} \\
		\gamma(0)=&-K_i \label{eq:bc_gamma}
	\end{align}
	where $K_i\in \mathbb{R}^{1\times 2}$ is a control gain to be designed, $\gamma(x)=[\gamma_1(x), \gamma_2(x)]$. Notice, the expressions of inverse kernels $\lambda(x)$, $\rho(x,y)$, $\sigma(x,y)$ are similar to those of $\gamma(x)$, $k(x,y)$, $l(x,y)$, where $k\in C^1(\mathcal D)$, $l \in C^1(\mathcal D)$, $\gamma\in C^1(0,1)$, and the detailed analysis is shown in \cite{bib0}.
	
	Using the transformations \eqref{Trans:beta1}, \eqref{Trans:xi1}, the original \eqref{eq:xi1}--\eqref{eq:X1} is converted to the following equations
	\begin{align}
		\sqrt{\epsilon_i}\beta_{i,t}(x,t)&=\beta_{i,x}(x,t)+G_i(x)\ddot{\theta}_{di}(t), \label{eq:beta1} \\
		\sqrt{\epsilon_i}\eta_{i,t}(x,t)&=-\eta_{i,x}(x,t)-\epsilon_i(1+R_i-x)\ddot{\theta}_{di}(t)
		\notag \\
		&-\frac{b_i^2}{2}\int_{0}^{x}\cosh(b_i(x-y))(\eta_i(y,t)-\beta_i(y,t))dy  \notag \\
		&-\frac{b_i^2}{2}\int_{0}^{x}\cosh(b_i(x-y))\lambda(y)dyX_i(t) \notag \\
		&+\frac{b_i^2}{2}\int_{0}^{x}\int_{y}^{x}\cosh(b_i(x-z))(\sigma(z,y)\eta_i(y,t)\notag \\
		&+\rho(z,y)\beta_i(y,t))dzdy, \label{eq:eta1_target} \\
		\eta_i(0,t)&=-\beta_i(0,t)+(C_i-K_i)X_i(t), \label{eq:bc_eta1_target} \\
		\dot{X}_i(t)&=(A_i+B_iK_i)X_i(t)+B_i\beta_i(0,t)+D_i\ddot{\theta}_{di}(t) \label{eq:X1_target}
	\end{align}
	where $G_i(x)=\sqrt{\epsilon_i}\gamma(x)D_i-\int_{0}^{x}(k(x,y)-l(x,y))dy$.
	
	Additionally, according to \cite{bib5} and recalling \eqref{eq:theta}, \eqref{Trans:beta1} and \eqref{eq:bc_xi1}, through a lengthy calculation involving a change in the order of integration within a double integral, the dynamic right boundary $\beta_{i,t}(1,t)$, originating from the ODE in the input channel, of \eqref{eq:beta1}--\eqref{eq:X1_target} is obtained as
	\begin{align}
		\beta_{i,t}(1,t)&=-\eta_{i,t}(1,t)+h_1\beta_i(1,t)+h_2\beta_i(0,t)+h_3\eta_i(1,t) \notag \\
		&+h_4\eta_i(0,t)+h_5X_i(t)+\int_{0}^{1}H_6(y)\beta_i(y,t)dy \notag \\
		&+\int_{0}^{1}H_7(y)\eta_i(y,t)dy+2\sqrt{\epsilon_i}R_iU_i(t) \label{bet}
	\end{align}
	where the gains $h_1$, $h_2$, $h_3$, $h_4$, $h_5$, $H_6$, $H_7$ are shown in Appendix\ref{h}.  Then choosing
	\begin{align}
		U_i(t)&=\frac{1}{2\sqrt{\epsilon_i}R_i}\bigg[ -(\acute{c}_i+h_1)\beta_i(1,t)+\eta_{i,t}(1,t)-h_2\beta_i(0,t) \notag \\
		&-h_3\eta_i(1,t)-h_4\eta_i(0,t)-\int_{0}^{1}H_6(y)\beta_i(y,t)dy \notag \\
		&-h_5X_i(t)-\int_{0}^{1}H_7(y)\eta_i(y,t)dy\bigg], \label{U_target}
	\end{align}
	we get the right (dynamic) boundary of the target system
	\begin{align}
		\beta_{i,t}(1,t)=-\acute{c}_i\beta_i(1,t)\label{eq:bc_beta1}
	\end{align}
	where $\acute{c}_i$ is a positive design parameter. Now we arrive at the target system \eqref{eq:beta1}--\eqref{eq:X1_target}, \eqref{eq:bc_beta1} via the proposed transformations and by the choice of the control input \eqref{U_target}.
	
	\begin{Rem}\label{Rem:state-feedback}
		Like \cite{bib5}, the right boundary of the hyperbolic PDEs and the input $\Delta \theta_i$-ODE are captured as a dynamic right boundary of the hyperbolic PDEs in the control design, and thus the target system \eqref{eq:beta1}--\eqref{eq:X1_target} is achieved with a dynamic boundary \eqref{eq:bc_beta1}, which encapsulates the dynamics of the input ODE.
	\end{Rem}
	
	Substituting \eqref{eq:eta1}, \eqref{Trans:beta1} and \eqref{eq:xi1} into \eqref{U_target}, we get the controller expressed by the original states
	\begin{align}
		&U_i(t)=n_1\xi_i(1,t)+n_2\eta_i(1,t)+n_3X_i(t)+n_4\xi_i(0,t) \notag \\
		&+n_5\eta_{i,x}(1,t)+\int_{0}^{1}N_6(y)\xi_i(y,t)dy+\int_{0}^{1}N_7(y)\eta_i(y,t)dy \label{U_origin}
	\end{align}
	where $n_1$, $n_2$, $n_3$, $n_4$, $n_5$, $N_6$, $N_7$ are given in Appendix\ref{n}.
	
	The properties of the state-feedback closed-loop system are presented as follows
	\begin{Thm}\label{Thm:state-feedback}
		If  initial value $(\Delta \theta_i(0),\Delta \dot{\theta}_i(0),\xi_i(x,0),$ $\eta_i(x,0), X_i(0)) \in \mathcal H$, the closed-loop system consists of the plant \eqref{eq:xi1}--\eqref{eq:X1} and the controller \eqref{U_origin} has the following properties:
		
		1) All states in the closed-loop system are uniformly ultimately bounded in the sense that there exist positive constants $\Upsilon_0,\acute{C}_0,\mathcal{C}_0$ given by \eqref{mathC0} such that
		\begin{align}
			\Omega_0(t)< \Upsilon_0\Omega_0(0)e^{-\acute{C}_0 t}+\mathcal{C}_0 \label{Omeg0}
		\end{align}
		where
		\begin{align}
			\Omega_0(t)=& | \Delta \theta_i(t)|^2+| \Delta \dot{\theta}_i(t)|^2+\| \xi_i(\cdot,t)\|^2_{H^1}+\| \eta_i(\cdot,t)\|^2_{H^1} \notag \\
			&+| X_i(t)|^2,
		\end{align}
		and $\acute{C}_0,\mathcal{C}_0$ can be adjusted by the design parameter $\acute c_i$, $K_i$ according to \eqref{eq:acute_lambda_1}, \eqref{eq:math_C1}.
		
		2) The control input $U_i(t)$ \eqref{U_origin} is   bounded.\end{Thm}
	\begin{proof}
		The proof of Theorem $\ref{Thm:state-feedback}$ is shown in Appendix\ref{Proof:Theorem1}.
	\end{proof}
	\subsection{Observer Design}
	Since the distributed states required by controller \eqref{U_origin} are often inaccessible in practice, in this section, we design an observer for the unmeasured states only using the available measurements at PDE boundaries.
	
	Relying on the measurements $\Delta \theta_i(t)$, $C_iX_i(t)$, $\xi_i(0,t)$, like \cite{bib8}, the observer is built as a copy of the plant \eqref{eq:theta}, \eqref{eq:xi1}--\eqref{eq:X1} with output error injections
	\begin{align}
		\sqrt{\epsilon_i}\hat{\xi}_{i,t}(x,t)=&\hat{\xi}_{i,x}(x,t)+\frac{b_i^2}{2}\int_{0}^{x}\cosh(b_i(x-y))(\hat{\xi}_i(y,t) \notag \\
		&-\hat{\eta}_i(y,t))dy-\epsilon_i(1+R_i-x)\ddot{\theta}_{di}(t) \notag \\
		&-\Gamma^\xi(x)(\hat{\xi}_i(0,t)-\xi_i(0,t)), \label{eq:observer_xi1} \\
		\sqrt{\epsilon_i}\hat{\eta}_{i,t}(x,t)=&-\hat{\eta}_{i,x}(x,t)-\epsilon_i(1+R_i-x)\ddot{\theta}_{di}(t)+\frac{b_i^2}{2} \notag \\
		&\int_{0}^{x}\cosh(b_i(x-y))(\hat{\xi}_i(y,t)-\hat{\eta}_i(y,t))dy \notag \\
		&-\Gamma^\eta(x)(\hat{\xi}_i(0,t)-\xi_i(0,t)),  \label{eq:observer_eta1} \\
		\hat{\xi}_i(1,t)=&-\hat{\eta}_i(1,t)+2\sqrt{\epsilon_i}R_i\Delta\dot{\theta}_i(t),\label{eq:observer_bc_xi1} \\
		\hat{\eta}_i(0,t)=&-\xi_i(0,t)+C_iX_i(t), \label{eq:observer_bc_eta1} \\
		\dot{\hat{X}}_i(t)=&A_i\dot{\hat{X}}_i(t)+B_i\xi_i(0,t)+D_i\ddot{\theta}_{di}(t) \notag \\
		&-L_xC_i(\hat{X}_i(t)-X_i(t)), \label{eq:observer_X1}
	\end{align}
	where the observer gain $L_x$
	is chosen such that $A_i-L_xC_i$is Hurwitz,
	and where observer gains $\Gamma^\xi(x)$, $\Gamma^\eta(x)$ \eqref{eq:Gamma_eta_xi} are  to be determined later.
	\begin{Rem} \label{Rem:observer}
		All the measurements used by the observer are available. The signal $\xi_i(0,t)$ can be obtained via $\upsilon_{i,t}(0,t)$, $\upsilon_{i,x}(0,t)$, $\theta_i(t)$, $\dot{\theta}_i(t)$, according to \eqref{varpi}--\eqref{det_theta1}, \eqref{xi1}--\eqref{acute x1}. Also, recalling \eqref{varpi}, \eqref{det_theta1} and \eqref{acute x1}, $C_iX_i(t)$, $\Delta \dot{\theta}_i(t)$ can be derived via $\upsilon_i(0,t)$, $\theta_i(t)$, $\dot{\theta}_i(t)$.
		In the experiment of the $i$-th link-joint, $\theta_i(t)$ is measured by an encoder associated with the $i$-th joint. According to \cite{bib9}, recalling \eqref{parameters_dimensionless}, the estimation of $\upsilon_i(0,t)$ is obtained via $\upsilon_i(0,t)=\frac{2E_{bi}(t)}{3L_{wi}^\ast}$ where $E_{bi}(t)$ is the base strain measured by a strain gauge offset potentiometer, which is placed at the proximal end, and where $L_{wi}^\ast$ is the flexible link thickness.
		Similarly, $\upsilon_{i,x}(0,t)$ is expressed as $\upsilon_{i,x}(0,t)=\frac{4E_{bi}(t)}{3L_w^\ast}$. The time derivatives of $\theta_i(t)$, $\upsilon_{i}(0,t)$, i.e., $\dot{\theta}_i(t)$, $\upsilon_{i,t}(0,t)$ are derived via using different second-order low-pass filters expressed as $\mathcal G(s)=\frac{\omega_n^2}{s^2+2\zeta\omega_n+\omega_n^2}$, where $\omega_n$ is the signal filter cutting frequency, $\zeta$ is the filter damping ratio. Therefore, in the establishment of the observer for nDSFMR, all sensing signals we need are the angle of the $i$-th joint $\theta_i(t)$ and the base strain $E_{bi}$ of the $i$-th link, $i=1,2,\cdots,n$.
	\end{Rem}
	
	Define the observer error as
	\begin{align}
		[\widetilde{X}_i(t),\widetilde{\xi}_i(x,t),\widetilde{\eta}_i(x,t)]&=[\hat{X}_i(t),\hat{\xi}_i(x,t),\hat{\eta}_i(x,t)] \notag \\
		&-[X_i(t),\xi_i(x,t),\eta_i(x,t)]. \label{parameters_observer_error}
	\end{align}
	
	According to \eqref{eq:observer_xi1}--\eqref{eq:observer_X1}, the observer error dynamics can be obtained as
	\begin{align}
		\sqrt{\epsilon_i}\widetilde{\xi}_{i,t}(x,t)=&\widetilde{\xi}_{i,x}(x,t)+\frac{b_i^2}{2}\int_{0}^{x}\cosh(b_i(x-y))(\widetilde{\xi}_i(y,t) \notag \\
		&-\widetilde{\eta}_i(y,t))dy-\Gamma^\xi(x)\widetilde{\xi}_i(0,t), \label{eq:xi1_observer_error} \\
		\sqrt{\epsilon_i}\widetilde{\eta}_{i,t}(x,t)=&-\widetilde{\eta}_{i,x}(x,t)+\frac{b_i^2}{2}\int_{0}^{x}\cosh(b_i(x-y))(\widetilde{\xi}_i(y,t) \notag \\
		&-\widetilde{\eta}_i(y,t))dy-\Gamma^\eta(x)\widetilde{\xi}_i(0,t),  \label{eq:eta1_observer_error} \\
		\widetilde{\xi}_i(1,t)=&-\widetilde{\eta}_i(1,t), \quad
		\widetilde{\eta}_i(0,t)=0, \label{eq:bc_xi1_eta1_observer_error} \\
		\dot{\widetilde{X}}_i(t)=&(A_i-L_xC_i)\dot{\widetilde{X}}_i(t). \label{eq:X1_observer_error}
	\end{align}
	
	To eliminate the in-domain coupling between PDEs in \eqref{eq:xi1_observer_error}, \eqref{eq:eta1_observer_error}, we introduce the invertible backsteppting transformations:
	\begin{align}
		\widetilde{\xi}_i(x,t)=&\widetilde{\beta}_i(x,t)+\int_{0}^{x}\psi(x,y)\widetilde{\beta}_i(y,t)dy, \label{Trans:xi1_observer} \\
		\widetilde{\eta}_i(x,t)=&\widetilde{\alpha}_i(x,t)+\int_{0}^{x}\phi(x,y)\widetilde{\beta}_i(y,t)dy \label{Trans:eta1_obsever}
	\end{align}
	where the kernel functions $\psi(x,y),\phi(x,y)$ on $\acute{\mathcal{D}}=\{0 \leq x \leq y \leq 1 \}$ satisfy
	\begin{align}
		&-\phi_x(x,y)+\phi_y(x,y)+\frac{b_i^2}{2}\cosh(b_i(x-y)) \notag \\
		&-\frac{b_i^2}{2}\int_{x}^{y}\cosh(b_i(x-z))(\phi(z,y)-\psi(z,y))dz=0, \label{eq:phi} \\
		&\psi_x(x,y)+\psi_y(x,y)+\frac{b_i^2}{2}\cosh(b_i(x-y)) \notag \\
		&-\frac{b_i^2}{2}\int_{x}^{y}\cosh(b_i(x-z))(\phi(z,y)-\psi(z,y))dz=0, \label{eq:psi} \\
		&\phi(x,x)=0,\quad \psi(1,y)=-\phi(1,y). \label{eq:bc_phi_psi}
	\end{align}
	
	Please note that the kernels satisfy $\phi\in C^1(\mathcal {\acute D})$, $\psi \in C^1(\mathcal {\acute D})$, whose detailed analysis is given in \cite{bib0}.
	Applying the transformations \eqref{Trans:xi1_observer}, \eqref{Trans:eta1_obsever}, choosing the observer gain $\Gamma^\eta(x)$ and $\Gamma^\xi(x)$ as
	\begin{align}
		\Gamma^\eta(x)=\phi(x,0), \quad
		\Gamma^\xi(x)=\psi(x,0), \label{eq:Gamma_eta_xi}
	\end{align}
	we convert \eqref{eq:xi1_observer_error}--\eqref{eq:bc_xi1_eta1_observer_error} to the target observer error system
	\begin{align}
		\sqrt{\epsilon_i}\widetilde{\alpha}_{i,t}(x,t)=&-\widetilde{\alpha}_{i,x}(x,t)-\int_{0}^{x}S_{11}(x,y)\widetilde{\alpha}_i(y,t)dy, \label{eq:alpha_observer}  \\
		\sqrt{\epsilon_i}\widetilde{\beta}_{i,t}(x,t)=&\widetilde{\beta}_{i,x}(x,t)-\int_{0}^{x}S_{21}(x,y)\widetilde{\alpha}_i(y,t)dy, \label{eq:beta_observer} \\
		\widetilde{\beta}_i(1,t)=&-\widetilde{\alpha}_i(1,t), \quad \widetilde{\alpha}_i(0,t)=0 \label{eq:bc_beta_alpha_observer}
	\end{align}
	where
	$S_{11}(x,y)=\frac{b_i^2}{2}\cosh(b_i(x-y))+\int_{y}^{x}S_{21}(z,y)\phi(x,z)dz,$
	$S_{21}(x,y)=\frac{b_i^2}{2}\cosh(b_i(x-y))+\int_{y}^{x}S_{21}(z,y)\psi(x,z)dz.$
	
	The properties of the observer error system are presented as follows.
	\begin{Thm}\label{Thm:observer}
		Considering the observer system \eqref{eq:xi1_observer_error}--\eqref{eq:X1_observer_error} with observer gains \eqref{eq:Gamma_eta_xi} and initial values $\hat{\xi}_i(x,0),\hat{\eta}_i(x,0)$, $\hat{X}_i(0) \in \mathcal H$, the observer error system \eqref{eq:alpha_observer}--\eqref{eq:bc_beta_alpha_observer}  are exponentially stable in the sense that there exist positive constants $\Upsilon_e,\acute{C}_e$ such that
		\begin{align}
			\Omega_e(t)\leq \Upsilon_e\Omega_e(0)e^{-\acute{C}_et} \label{Ome}
		\end{align}
		where $\Omega_e(t)=\| \widetilde{\eta}_i(\cdot,t)\|^2_{H^1}+\| \widetilde{\xi}_i(\cdot,t)\|^2_{H^1}+| \widetilde{X}_i(t)|^2$ and $\acute{C}_e$ can be adjusted by the design parameter $L_x$.
	\end{Thm}
	
	\begin{proof}
		It is straightforward to get Theorem \ref{Thm:observer} through the same steps in the proof of Theorem \ref{Thm:state-feedback}. Please see more details in \cite{bib0}.
	\end{proof}
	
	\subsection{Output-Feedback Closed-Loop System}\label{outputfd}
	
	Substituting all the unmeasurable states in the state feedback control law \eqref{U_origin} with the observer states, the output-feedback controller can be obtained as
	
	\begin{align}
		&U_{ofi}(t)=n_1\hat{\xi}_i(1,t)+n_2\hat{\eta}_i(1,t)+n_3\hat X_i(t)+n_4 \xi_i(0,t) \notag \\
		&+n_5\hat \eta_{i,x}(1,t)+\int_{0}^{1}N_6(y)\hat \xi_i(y,t)dy+\int_{0}^{1}N_7(y)\hat \eta_i(y,t)dy. \label{U_output_fedback}
	\end{align}
	
	The properties of the closed-loop system are given as follows:
	\begin{Thm}\label{Thm:output_feedback}
		For all initial data  $(\Delta \theta_i(0),\Delta \dot{\theta}_i(0),\xi_i(x,0),$ $\eta_i(x,0), X_i(0)) \in \mathcal H$
		and $(\hat{\xi}_i(x,0),\hat{\eta}_i(x,0), \hat{X}_i(0)) \in \mathcal H$, the output-feedback closed-loop system consisting of the plant \eqref{eq:det_theta1}, \eqref{eq:xi1}--\eqref{eq:X1}, the observer \eqref{eq:observer_xi1}--\eqref{eq:observer_X1}, and the controller \eqref{U_output_fedback},  has the following properties:
		
		1) There exists a positive constant $\Upsilon_{a}$ such that
		\begin{align}
			\Omega_{a}(t)\leq& \Upsilon_{a}{\Omega_{a}(0)}e^{-\acute{C}_{a}t}+\mathcal{C}_{a} \label{eq:Omega_a}
		\end{align}
		where
		\begin{align}
			&\Omega_{a}(t)=\| \eta_i(\cdot,t)\|^2_{H^1}+\| \hat{\eta}_i(\cdot,t)\|^2_{H^1}+\| \xi_i(\cdot,t)\|^2_{H^1}+\| \hat{\xi}_i(\cdot,t)\|^2_{H^1} \notag \\
			&+| \Delta \theta_i(t)|^2+| \Delta \dot{\theta}_i(t)|^2 +| X_i(t)|^2+| \hat{X}_i(t)|^2, \label{eq:Omega_a1}
		\end{align}
		and $\acute{C}_{a}=\min\{ \acute{C}_{0}, \acute{C}_{e} \}$ can be tuned via the parameters $\acute c_i$, $K_i$, $L_x$ as given in \eqref{eq:acute_lambda_1}, \eqref{eq:math_C1}, and where  $\mathcal{C}_{a}=2\mathcal{C}_0$ is positive, and $\mathcal C_0$ can be adjusted by $\acute c_i$ as given in \eqref{mathC0}.
		
		2) The output-feedback control law \eqref{U_output_fedback} is  bounded.
	\end{Thm}
	
	\begin{proof}
		The proof of Theorem \eqref{Thm:output_feedback} is stated in \cite{bib0}.
	\end{proof}
	
	\begin{Cor} \label{Cor:Thm3}
		Consider the practical system described by \eqref{eq:varpi}--\eqref{eq:bc_Phi} under the actual control input $\tau_{mi}$ obtained from \eqref{eq:Ui1}.
		The tracking error $\Delta\theta_i(t)$ \eqref{det_theta1}, vibration state $\varpi_i(x,t)$ \eqref{varpi}, $\Phi_i(x,t)$ \eqref{solution:Phi}, and their corresponding observer estimates $\hat{\varpi}_i(x,t)$, $\hat{\Phi}_i(x,t)$ are bounded in the sense that there exist positive constants $\Upsilon_d$, $\acute C_d$, $\mathcal C_d$ such that
		\begin{align}
			\Omega_d(t)\leq \Upsilon_de^{-\acute C_dt}\Omega_d(0)+\mathcal C_d \label{eq:Omega_d}
		\end{align}
		where
		\begin{align}
			&\Omega_d(t)=| \Delta \theta_i(t)|^2+| \Delta \dot{\theta}_i(t)|^2+\|\varpi_{i,t}(\cdot,t)\|^2+\|\varpi_{i,x}(\cdot,t)\|^2 \notag \\
			&+\|\Phi_{i,t}(\cdot,t)\|^2+\|\Phi_{i,x}(\cdot,t)\|^2+|\varpi_{i}(0,t)|^2+\|\hat{\varpi}_{i,t}(\cdot,t)\|^2 \notag \\
			&+\|\hat{\varpi}_{i,x}(\cdot,t)\|^2+\|\hat{\Phi}_{i,t}(\cdot,t)\|^2+\|\hat{\Phi}_{i,x}(\cdot,t)\|^2+|\hat{\varpi}_{i}(0,t)|^2, \label{eq:Omega_d1}
		\end{align}
		and where the ultimate bound $\mathcal C_d$ depends on the reference angle $\theta_{di}(t)$ \eqref{eq:conditons_theta_id}, while both $\mathcal C_d$ and decay rate $\acute C_d$ remain adjustable by the design parameters.
	\end{Cor}
	\begin{proof}
		Recalling \eqref{xi1}--\eqref{acute x1}, \eqref{solution:Phi}, and Theorem \ref{Thm:output_feedback}, we get $\acute \theta_{1,1} \Omega_d(t) \leq  \Omega_a(t) \leq \acute \theta_{1,2} \Omega_d(t)$ for some positive $\acute \theta_{1,1}$, $\acute \theta_{1,2}$. Then, according to \eqref{eq:Omega_a} we obtain that
		$
		\Omega_d \leq \frac{\acute \theta_{1,2}\Upsilon_{a}}{\acute \theta_{1,1}}{\Omega_{d}(0)}e^{-\acute{C}_{a}t}+\frac{\mathcal{C}_{a}}{\acute \theta_{1,1}}.
		$
		
		Thus, \eqref{eq:Omega_d} is achieved with $\Upsilon_d=\frac{\acute \theta_{1,2}\Upsilon_{a}}{\acute \theta_{1,1}}$, $\acute C_d=\acute C_a$, and $\mathcal C_d=\frac{\mathcal{C}_{a}}{\acute \theta_{1,1}}$. The proof of Corollary \ref{Cor:Thm3} is complete.
	\end{proof}
	
	\begin{table}[htbp]
		\centering
		\caption{Parameters of the Quanser 2DSFL Robot}
		\label{tab:parameters}
		\begin{tabular}{ccc}
			\hline
			Description & Value & Unit \\
			\hline
			Young's modulus $E^\ast$    & $200$  & $Gpa$ \\
			Modulus of rigidity $G^\ast$ & $77.5$  & $Gpa$  \\
			Density $\rho^\ast$   & 7833 & $kg/m^3$  \\
			Cross-sectional area (Link 1) $A_1^\ast$ & $9.677\times10^{-5}$ & $m^2$ \\
			Cross-sectional area (Link 2) $A_2^\ast$ & $3.871\times10^{-5}$ & $m^2$ \\
			Area moment of inertia (Link 1) $I_1^\ast$ & $1.873\times10^{-7}$ &  $m^4$ \\
			Area moment of inertia (Link 2) $I_2^\ast$ & $1.638\times10^{-8}$ &  $m^4$ \\
			Shear coefficient $k'$ & $0.53066$ & $-$ \\
			First natural frequency $\omega_0^\ast$ & $1797.07$ & $-$ \\
			Length of Link 1 $L_1^\ast$ & $0.195$ & $m$ \\
			Length of Link 2 $L_2^\ast$ & $0.195$ & $m$ \\
			Radius of Disk 1 $R_1^\ast$ & $0.085$ & $m$ \\
			Radius of Disk 2 $R_2^\ast$ & $0.07$ & $m$ \\
			Inertia moment (Disk 1) $J_1^{\ast}$ & $0.01$ & $kg\cdot m^2$ \\
			Inertia moment (Disk 2) $J_2^{\ast}$ & $0.009$ & $kg\cdot m^2$ \\
			Tip mass (Link 1) $m_1^\ast$   & $1.57$ & $kg$  \\
			Tip mass (Link 2) $m_2^\ast$   & $0.157$ & $kg$  \\
			Damping coefficient (Motor 1) $c_1^\ast$ & $-4$ & $--$  \\
			Damping coefficient (Motor 2) $c_2^\ast$ & $-1.5$ & $--$  \\
			Torque constant (Motor 1) $K_{t1}$ & $8.925$ & $N\cdot m/A$  \\
			Torque constant (Motor 2) $K_{t2}$ & $0.87$ & $N\cdot m/A$  \\
			\hline
		\end{tabular}
	\end{table}
	
	\begin{figure*}[htbp] 
		\centering
		\includegraphics[width=\textwidth]{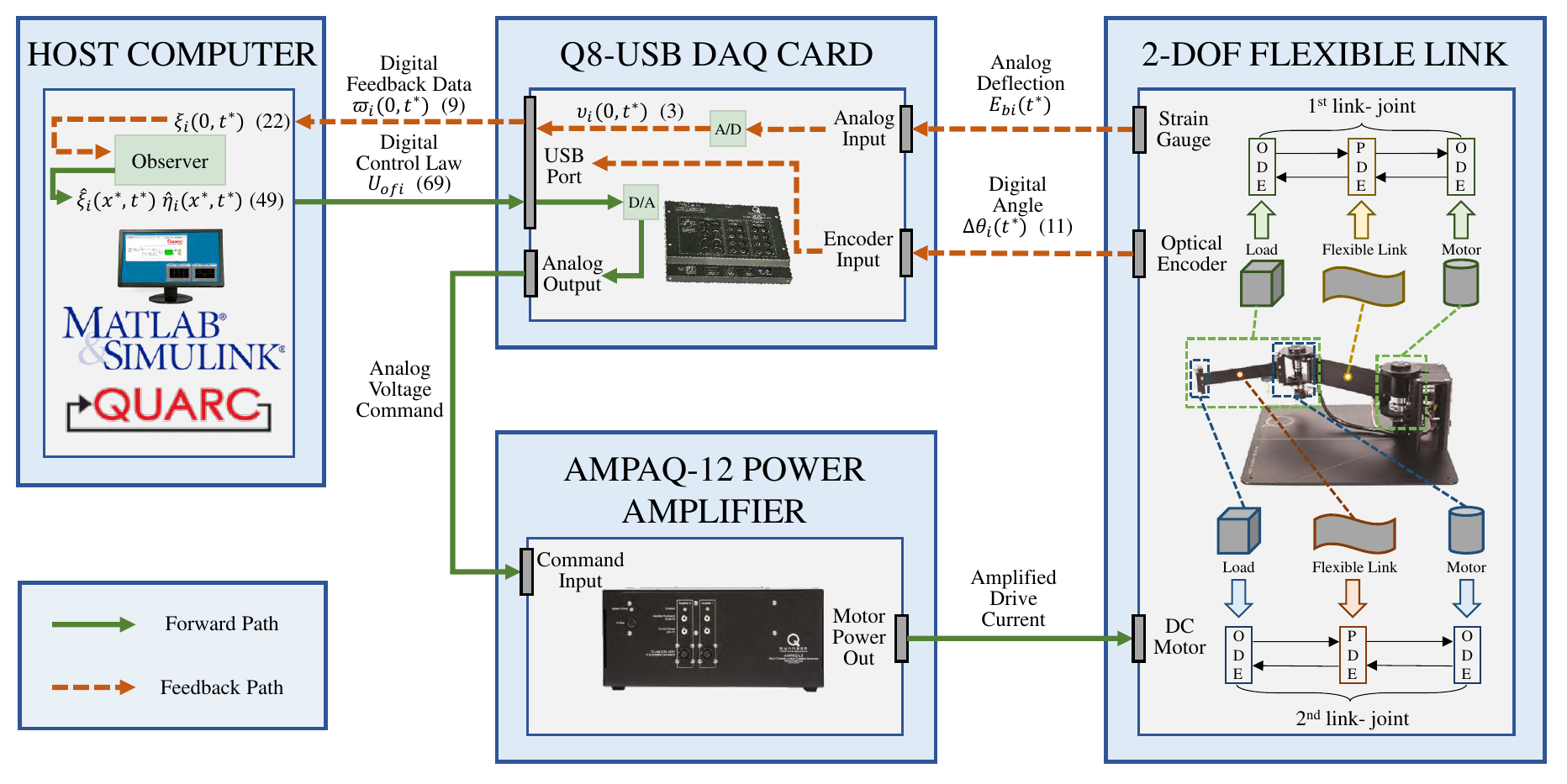}
		\caption{Hardware components of the Quanser 2DSFMR and integrated output-feedback control architecture for experimental implementation.}
		\label{fig:2DSFL Robot}
	\end{figure*}
	
	\begin{table*}[ht]
		\centering
		\caption{Performance of the proposed observer in estimating the slope of the transverse displacement $\varpi_{1,x^\ast}(x^\ast,t^\ast)$ at $x^\ast=0$ and $x^\ast=0.5L_1^\ast$. The first two rows display the observer estimates (blue), the real values (red), and the estimation error (green) in three reference profiles. The corresponding statistical analysis for each scenario is provided below, with Root-Mean-Square Error (RMSE), Mean Absolute Error (MAE), and Maximum Error (ME).}
		\renewcommand{\arraystretch}{1.5}
		\setlength{\tabcolsep}{10pt} 
		\begin{tabular}{@{} l ccc @{}}
			\toprule
			\textbf{Reference} & \textbf{Sinusoidal} & \textbf{Square} & \textbf{Sawtooth} \\
			\textbf{trajectory} & {\small Amp: $\frac{40\pi}{180}$, Freq: $0.2$Hz} & {\small Amp: $\frac{35\pi}{180}$, Freq: $0.1$Hz} & {\small Amp: $\frac{35\pi}{180}$, Freq: $0.2$Hz} \\
			\midrule
			
			\makecell[l]{\textbf{Slope of link} \\ \textbf{at} $x^\ast=0$} &
			\includegraphics[width=0.25\textwidth, valign=m]{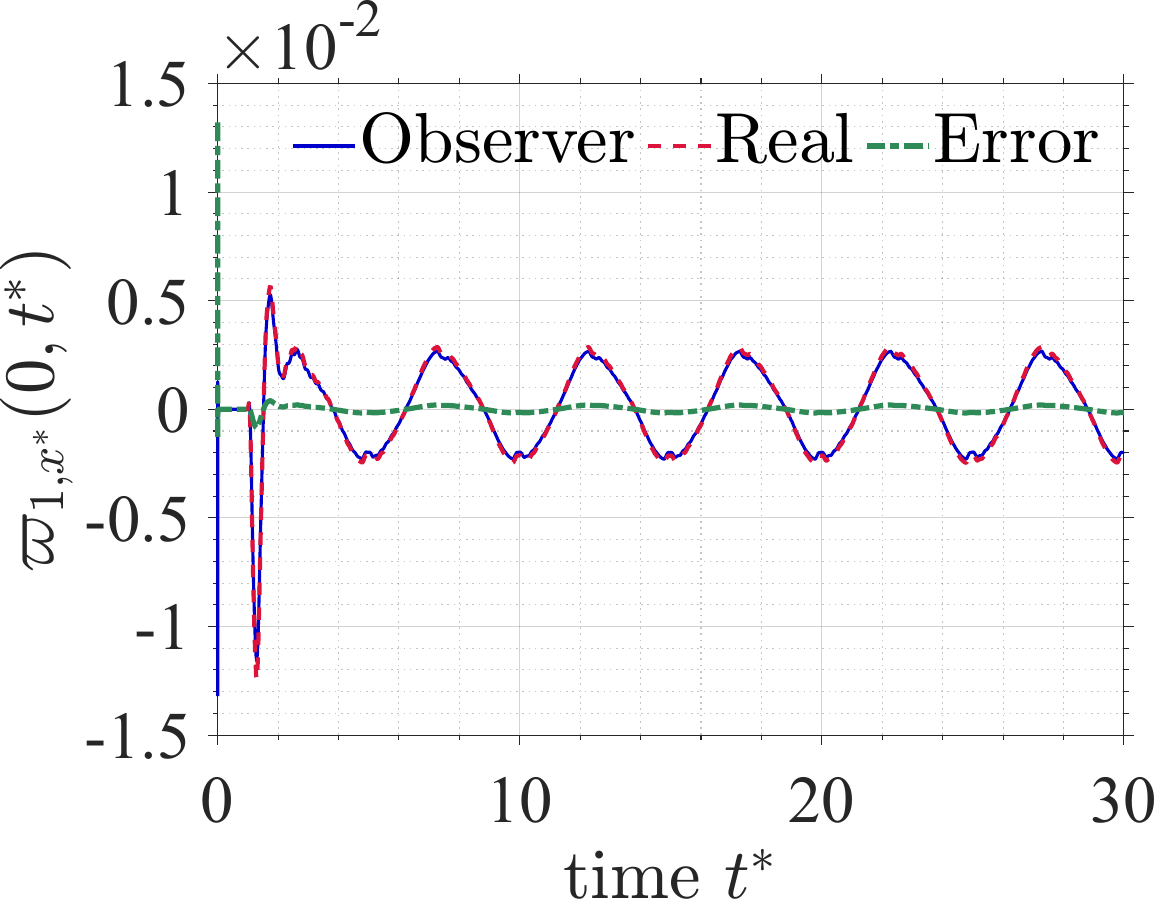} &
			\includegraphics[width=0.25\textwidth, valign=m]{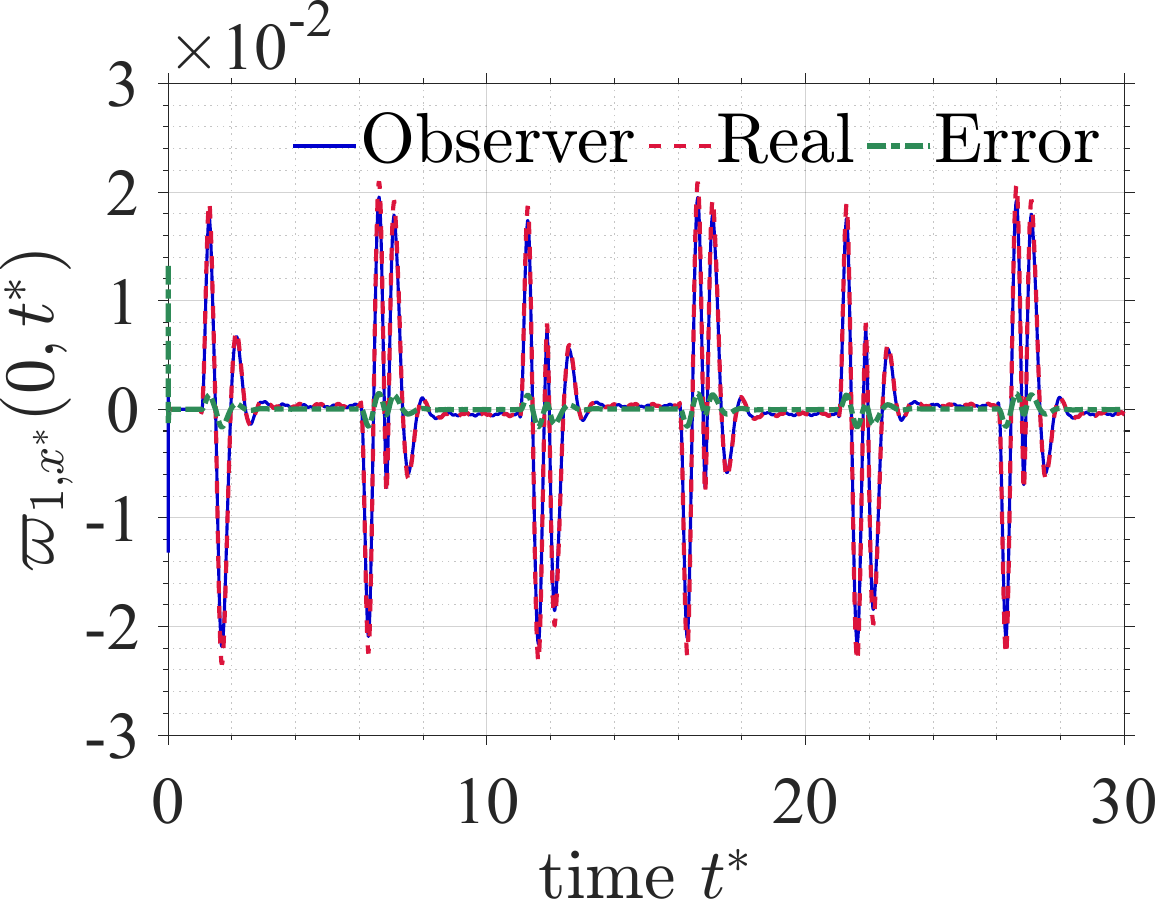} &
			\includegraphics[width=0.25\textwidth, valign=m]{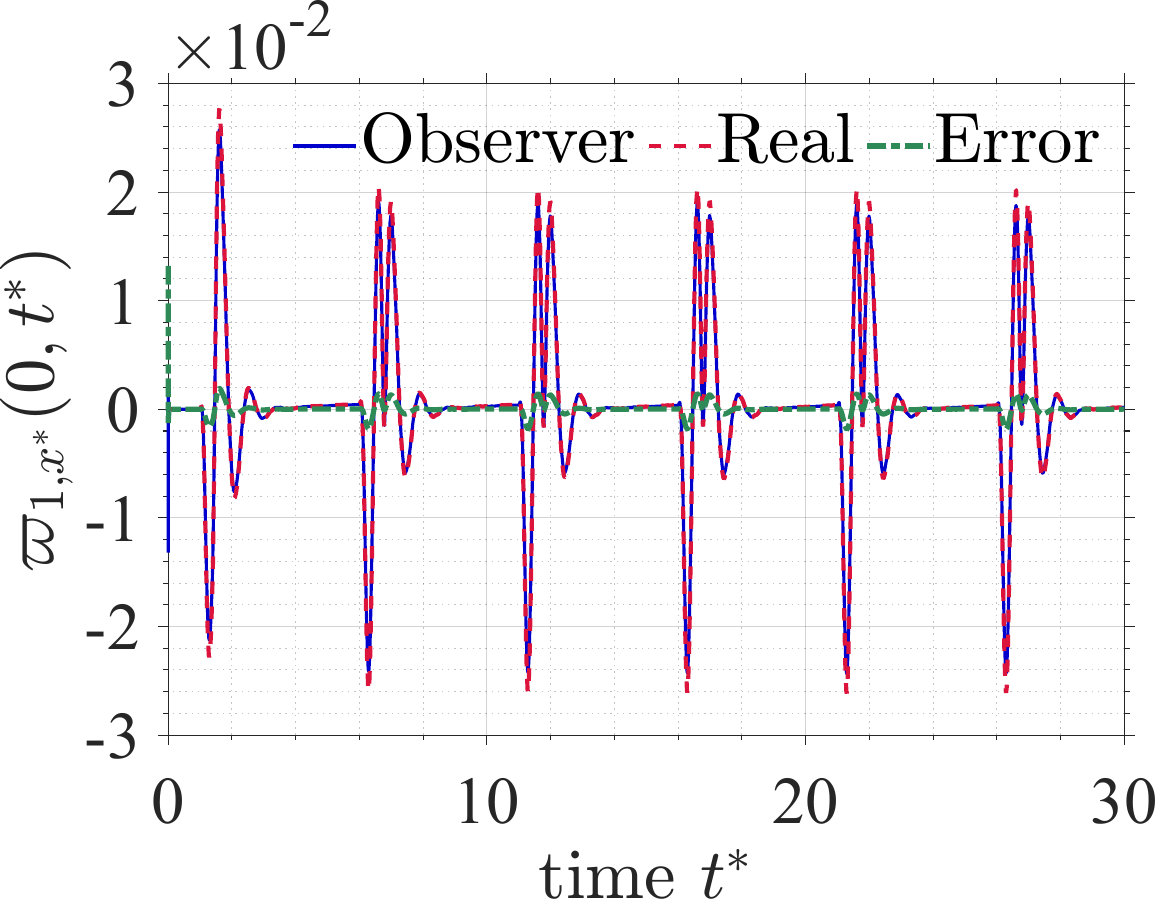} \\
			\addlinespace[3pt]
			
			\makecell[l]{\textbf{Slope of link} \\ \textbf{at} $x^\ast=0.5L_1^\ast$} &
			\includegraphics[width=0.25\textwidth, valign=m]{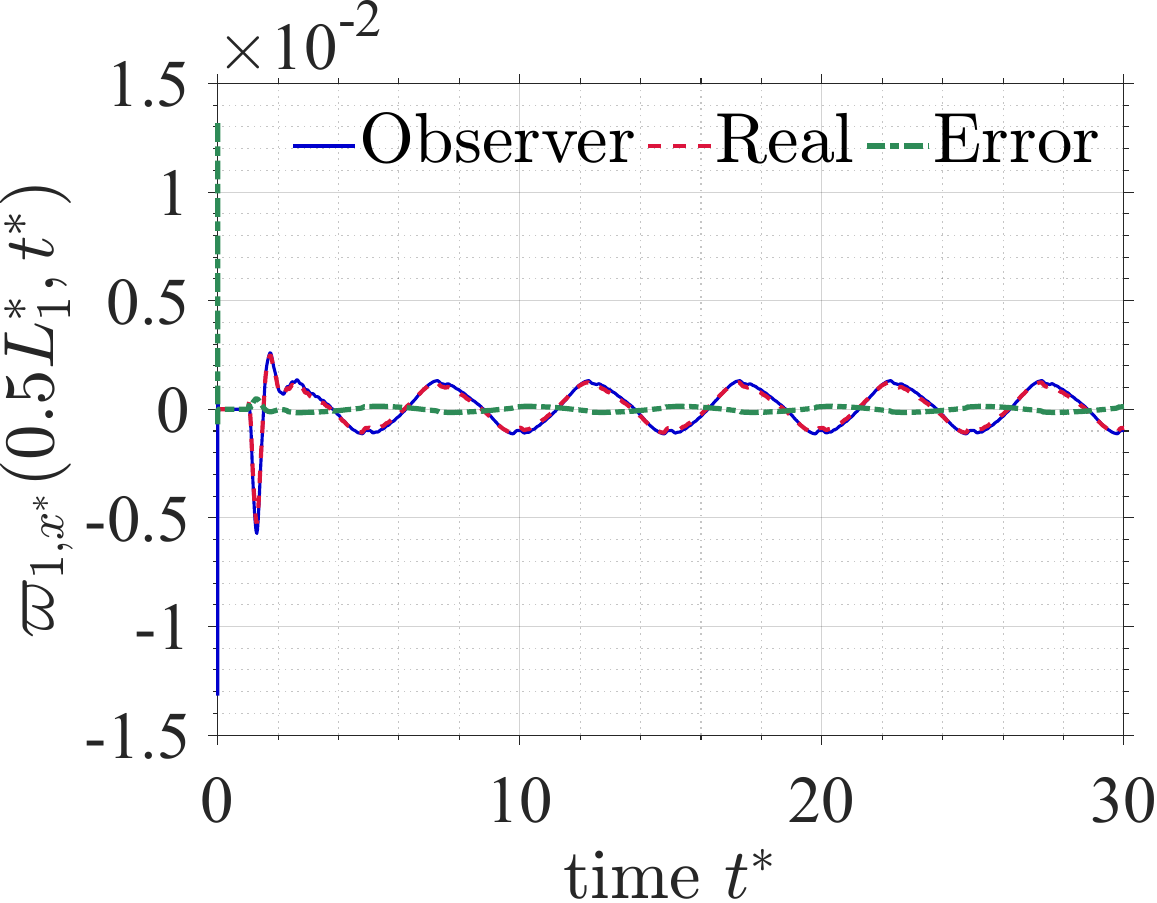} &
			\includegraphics[width=0.25\textwidth, valign=m]{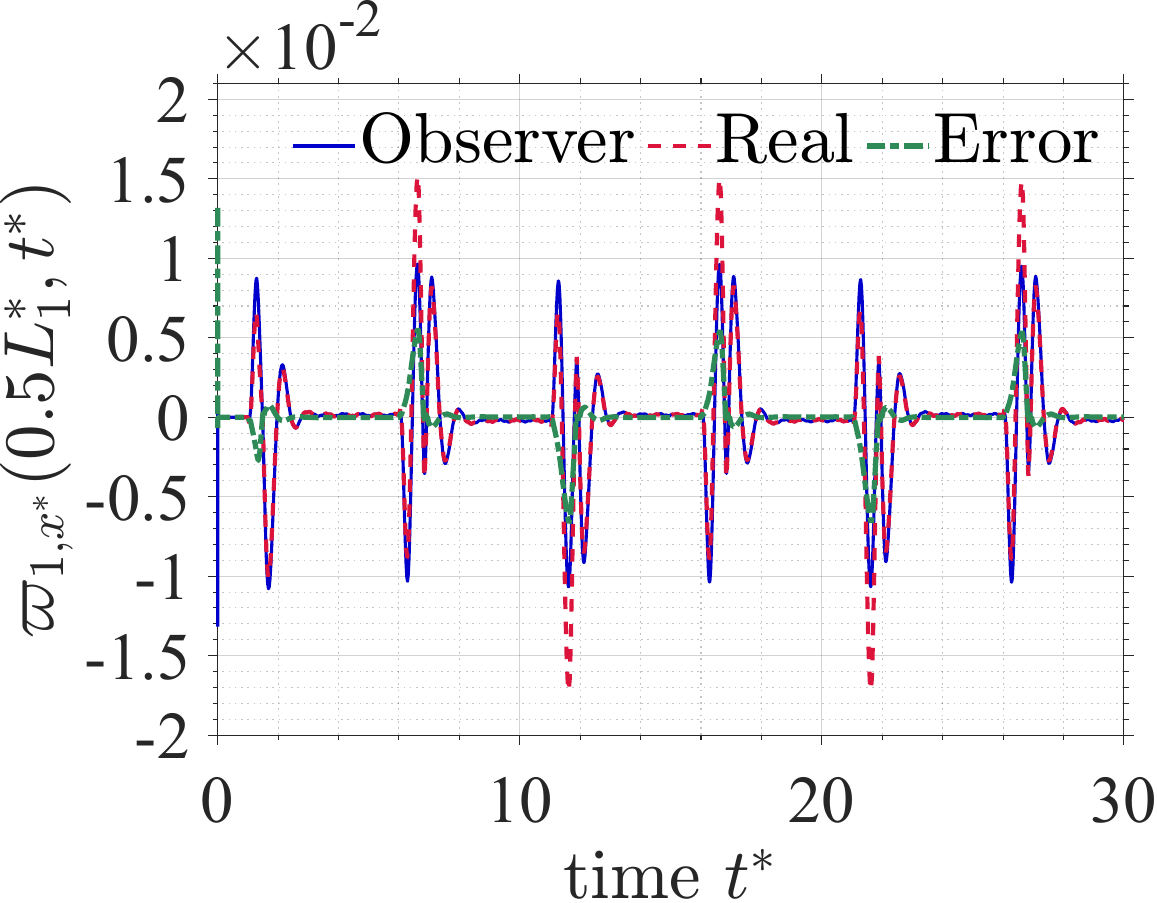} &
			\includegraphics[width=0.25\textwidth, valign=m]{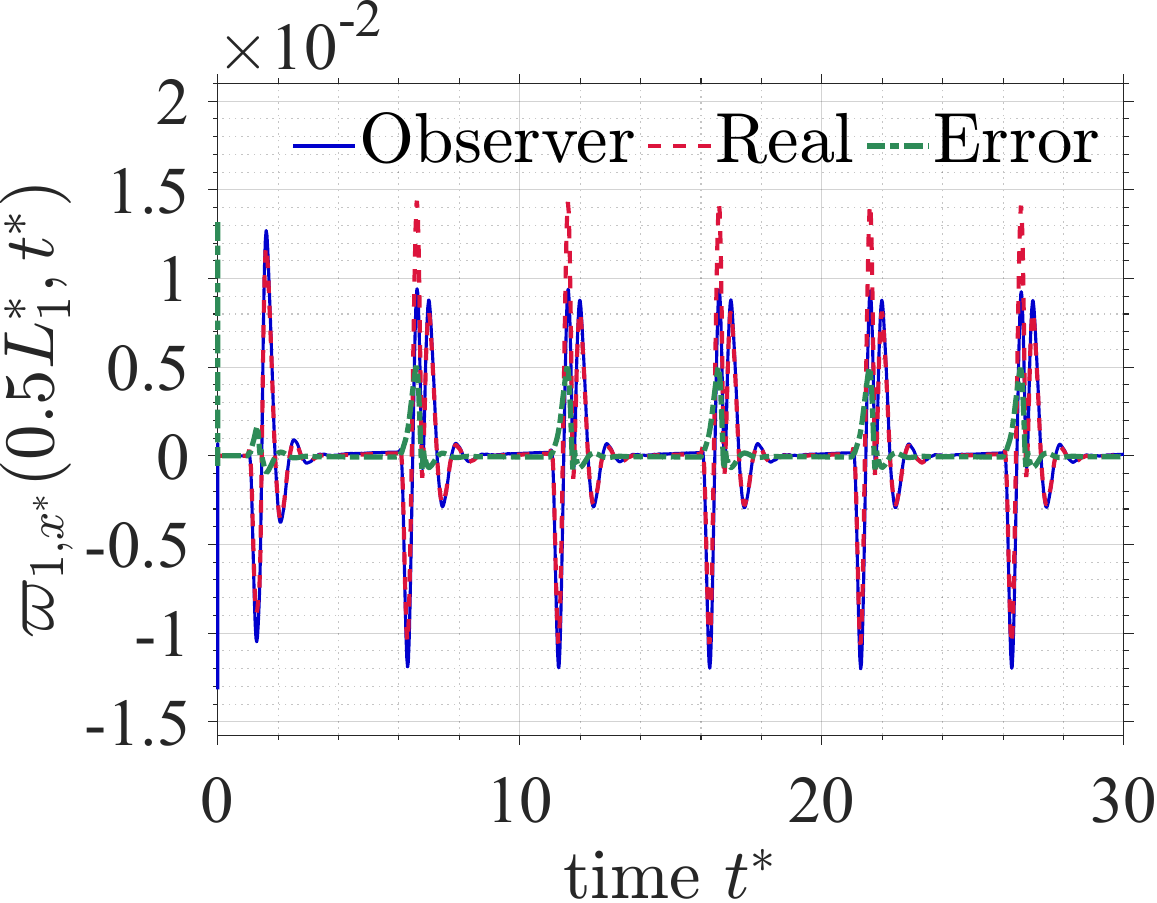} \\
			\addlinespace[3pt]
			
			\midrule			
			{\footnotesize
				\begin{tabular}{@{} l @{}}
					\textbf{Error index} \\
					\midrule
					$\varpi_{1,x^\ast}(0,t^\ast)$ \\
					$\varpi_{1,x^\ast}(0.5,t^\ast)$ \\
			\end{tabular}} &
			{\footnotesize
				\begin{tabular}{@{} ccc @{}}
					ME & RMSE & MAE \\
					\midrule
					0.01321 & 0.00016 &  0.00012 \\
					0.01318 & 0.00013 &  0.00009 \\
			\end{tabular}} &
			{\footnotesize
				\begin{tabular}{@{} ccc @{}}
					ME & RMSE & MAE \\
					\midrule
					0.01321 & 0.00048 & 0.00024 \\
					0.01318 & 0.00130 & 0.00047 \\
			\end{tabular}} &
			{\footnotesize
				\begin{tabular}{@{} ccc @{}}
					ME & RMSE & MAE \\
					\midrule
					0.01321 & 0.00050 & 0.00024 \\
					0.01318 & 0.00099 & 0.00037 \\
			\end{tabular}} \\
			\bottomrule
		\end{tabular}
		\label{Tab:observer_performance}
	\end{table*}
	
	\section{Experiment on a 2DSFMR} \label{Experiment}
	To demonstrate the practical effectiveness of the proposed control strategies, this section presents experimental results obtained using a $2$-degree-of-freedom serial flexible manipulator robot (2DSFMR). An accompanying experimental video is available at: https://valenhyl.github.io/Tmech/.
	
	First, we detail the hardware configuration of the 2DSFMR prototype. Next, the proposed output-feedback backstepping controller developed in Sec. \ref{Modeling} is validated on the primary link-joint, and the integral robot, in comparison with a traditional LQR method with feedforward \cite{bib9} given by
	\begin{align}
		U_{lqri}=&\acute k_1\Delta \theta_i(t)+ \acute k_2\Delta \dot{\theta}_i(t)+ \acute k_3\upsilon_i(0,t)+ \acute k_4\upsilon_{i,t}(0,t) \notag \\
		&+\acute k_5 \theta_{di}(t)+\acute k_6 \dot{\theta}_{di}(t)
	\end{align}
	where $\acute k_1$, $\acute k_2$, $\acute k_3$, $\acute k_4$ and $\acute k_5$, $\acute k_6$ represent the LQR feedback and feedforward gains, respectively, which are derived from the system parameters and the chosen weight matrices.
	
	\subsection{Quanser 2DSFMR}
	Fig. \ref{fig:2DSFL Robot} illustrates the hardware components of the Quanser 2DSFMR utilized in this study.
	The two joints are actuated by a Maxon 273759 precious brush motor (90 W) and a Maxon 118752 precious brush motor (20 W), respectively. Both motors are equipped with harmonic gearheads to eliminate mechanical backlash.
	The primary flexible link (Link 1) has a width of 0.0762 $m$, and a thickness of 0.00127 $m$, while the secondary link (Link 2) has a width of 0.0381 $m$, and a thickness of 0.000889 $m$.
	Moreover, the dynamic model parameters of the primary link (link 1) and the secondary link (link 2), derived from direct measurement and the user manual given by Quanser, are listed in Tab. \ref{tab:parameters}. Additionally, an external power supply powers the two joint position limit switches and the two strain gauge sensors.
	
	Recalling Remark \ref{Rem:observer}, the requisite physical signals, i.e. $\theta_1(t^\ast)$, $E_{b1}(t^\ast)$, $\theta_2(t^\ast)$, $E_{b2}(t^\ast)$, are acquired using two quadrature optical encoders for the digital angular positions of the motors and a strain gauge sensor mounted at the clamped base of each flexible link, with the latter connected to onboard amplifier circuits.
	
	The integrated output-feedback control architecture for experimental implementation is also depicted in Fig. \ref{fig:2DSFL Robot}. The hardware interface relies on a Quanser AMPAQ current amplifier and a Q8-USB data acquisition (DQA) device. The proposed output-feedback controller $\tau_{m1}$, $\tau_{m2}$ obtained from \eqref{eq:Ui1}, \eqref{U_output_fedback} and the corresponding observers obtained from \eqref{eq:observer_xi1}--\eqref{eq:observer_X1} are implemented in MATLAB/Simulink. During operation, real-time sensor data is captured via the Q8-USB DQA to feed the observers, while the computed control laws are converted into analog current signals to drive the motors via the AMPAQ amplifier.
	Moreover, each experimental trial runs for a total duration of 31s, with the initial 1$s$ providing a buffer for sensor error compensation.
	\begin{figure*}[ht]
		\centering
		
		\begin{subfigure}{\textwidth}
			\centering
			\begin{subfigure}[b]{0.3\textwidth}
				\centering
				\includegraphics[width=\textwidth]{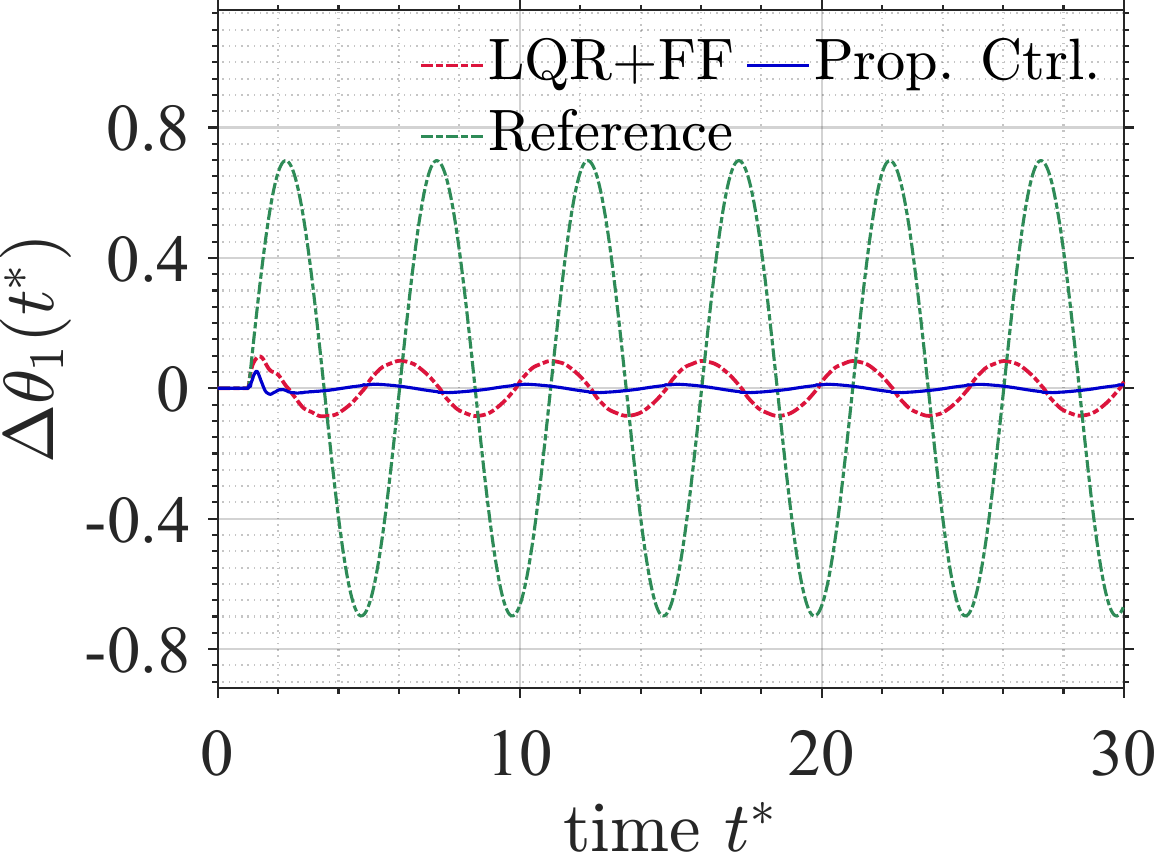}
			\end{subfigure}
			\hfill
			\begin{subfigure}[b]{0.3\textwidth}
				\centering
				\includegraphics[width=\textwidth]{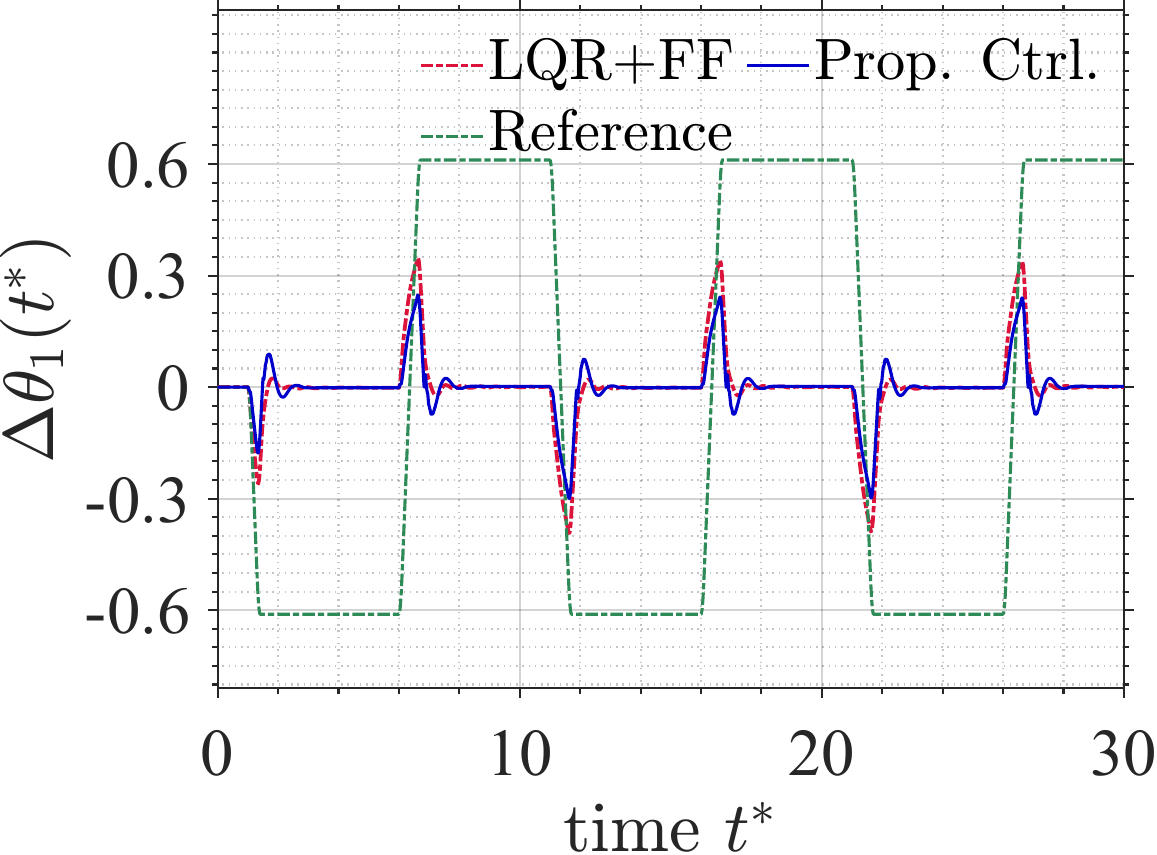}
			\end{subfigure}
			\hfill
			\begin{subfigure}[b]{0.3\textwidth}
				\centering
				\includegraphics[width=\textwidth]{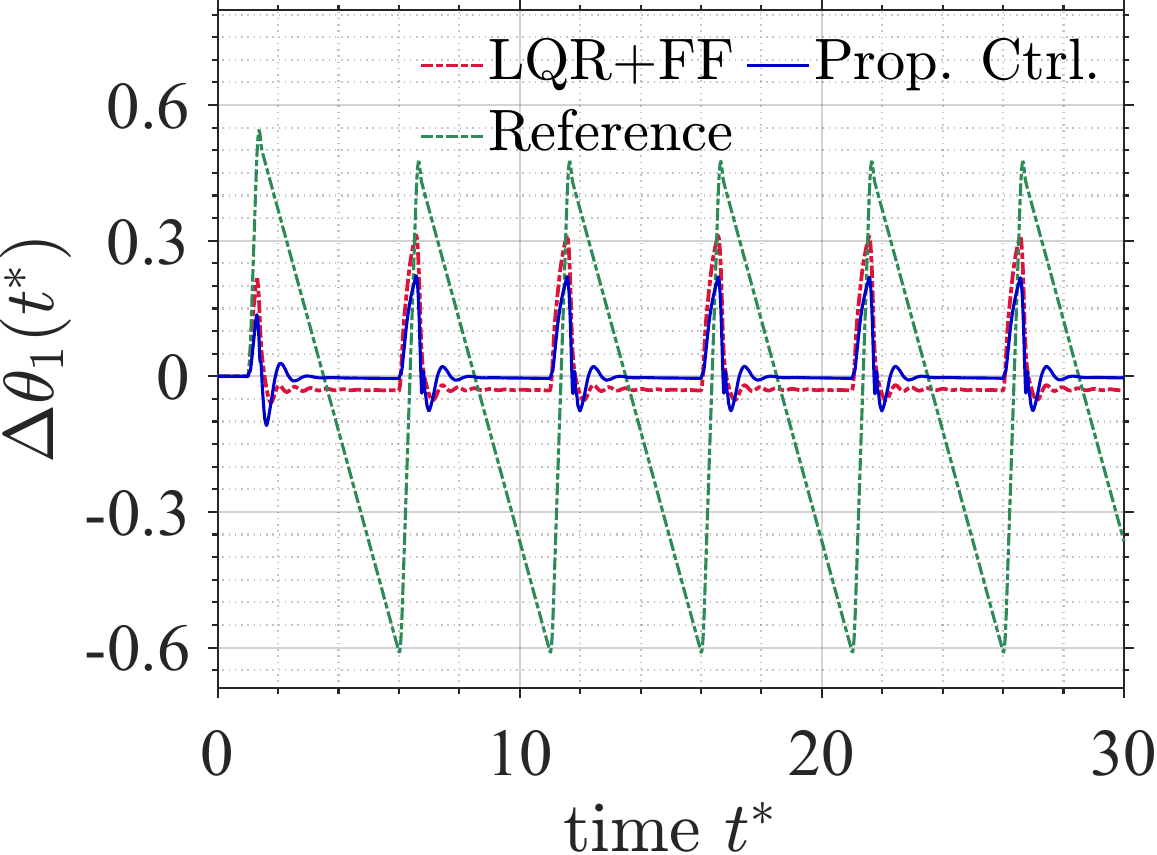}
			\end{subfigure}
			\caption{Angular tracking error $\Delta \theta_1(t^\ast)$.}
			\label{fig:angular_error}
		\end{subfigure}
		
		\vspace{1.5em} 
		
		\begin{subfigure}{\textwidth}
			\centering
			\begin{subfigure}[b]{0.3\textwidth}
				\centering
				\includegraphics[width=\textwidth]{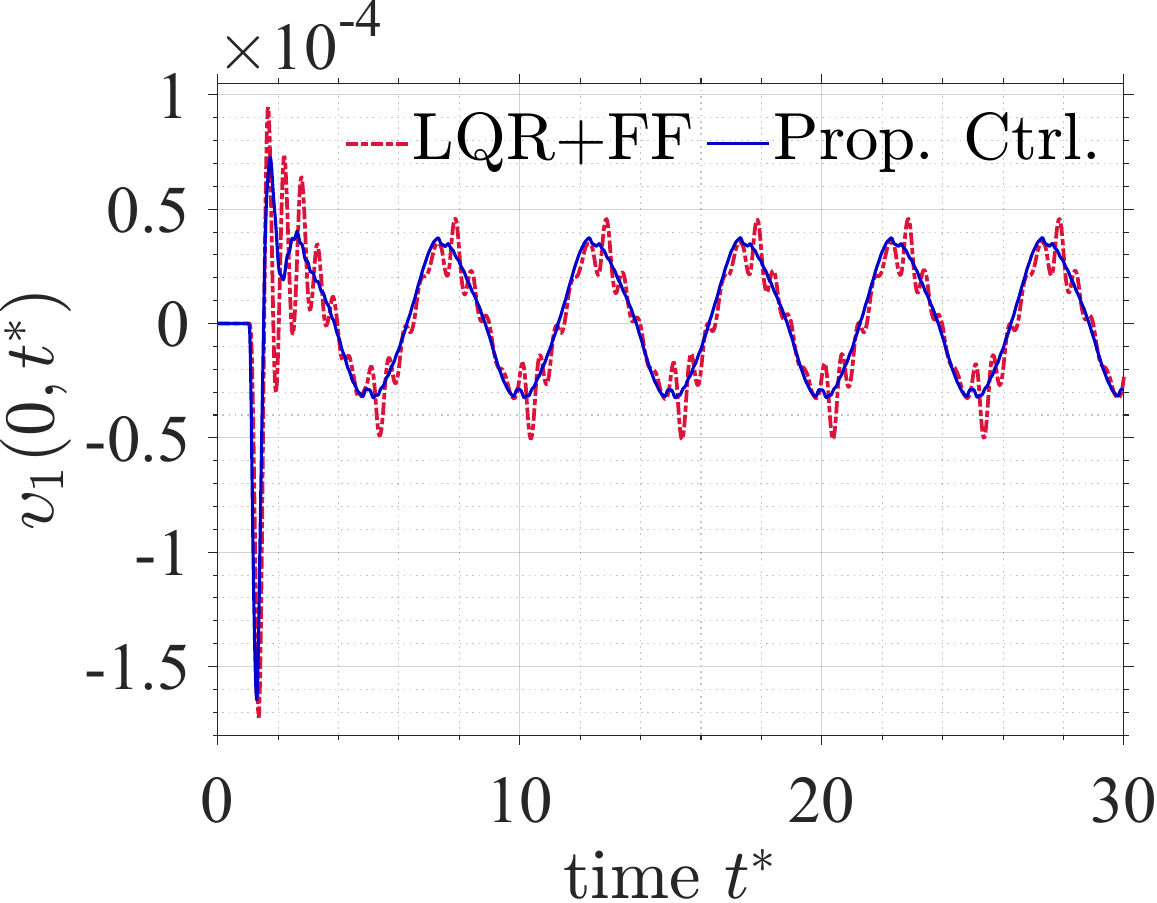}
			\end{subfigure}
			\hfill
			\begin{subfigure}[b]{0.3\textwidth}
				\centering
				\includegraphics[width=\textwidth]{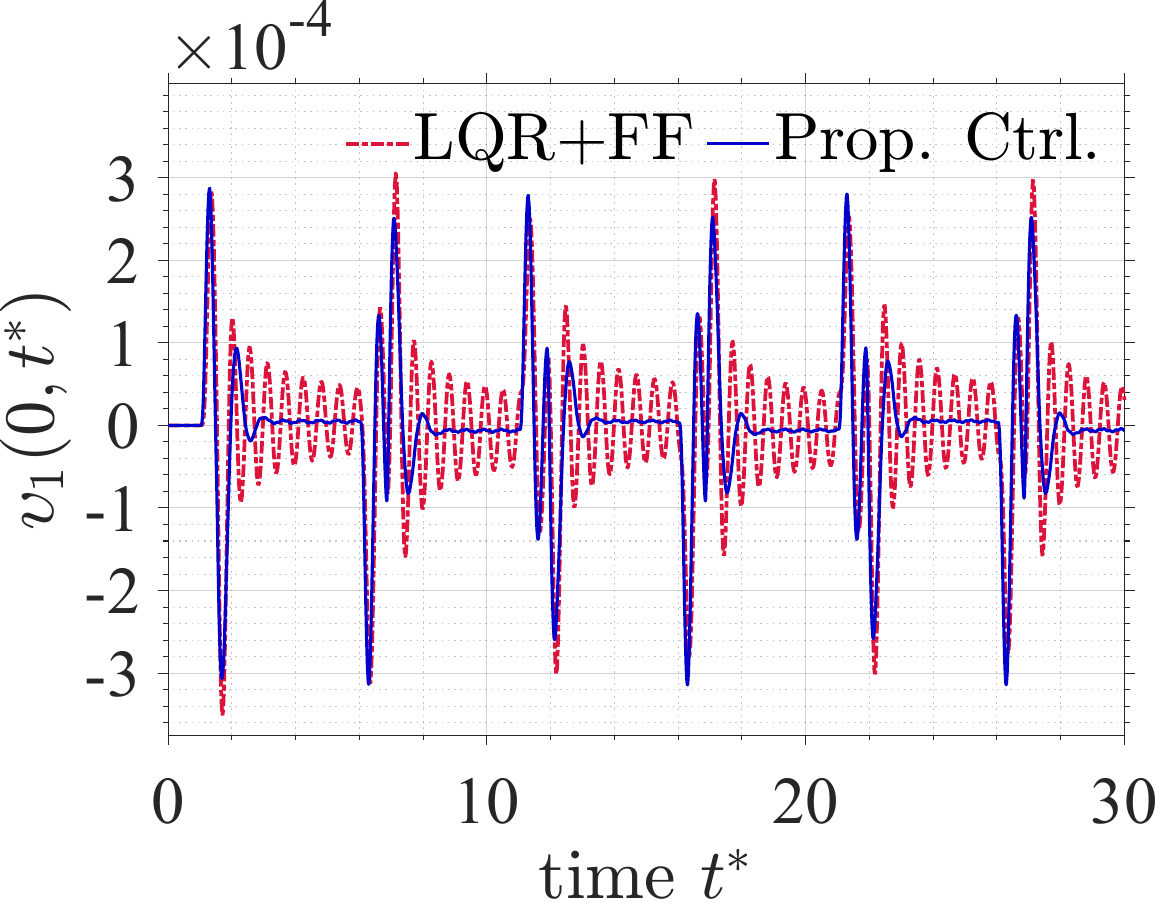}
			\end{subfigure}
			\hfill
			\begin{subfigure}[b]{0.3\textwidth}
				\centering
				\includegraphics[width=\textwidth]{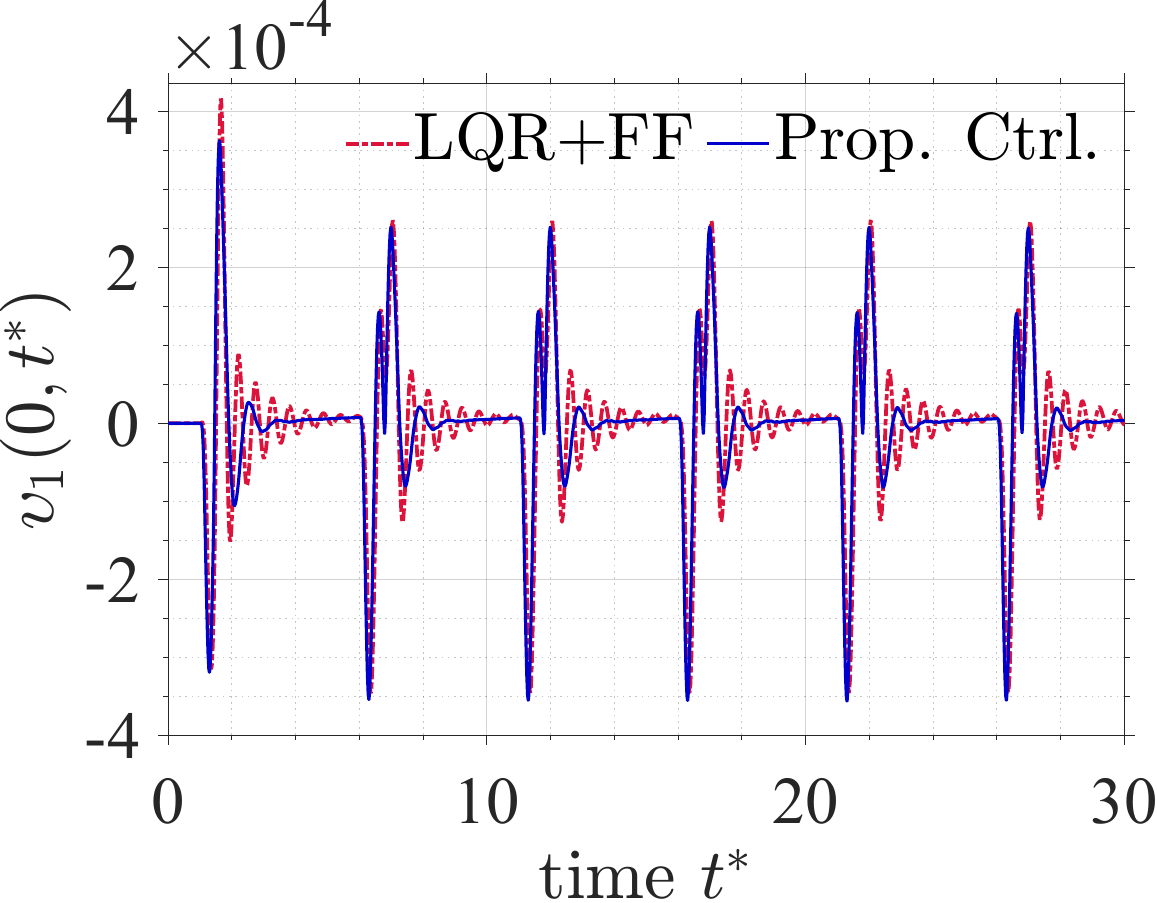}
			\end{subfigure}
			\caption{Transverse displacement $\upsilon_1(1,t^\ast)$.}
			\label{fig:transverse_disp}
		\end{subfigure}
		
		\vspace{1.5em}
		
		\begin{subfigure}{\textwidth}
			\centering
			\begin{subfigure}[b]{0.3\textwidth}
				\centering
				\includegraphics[width=\textwidth]{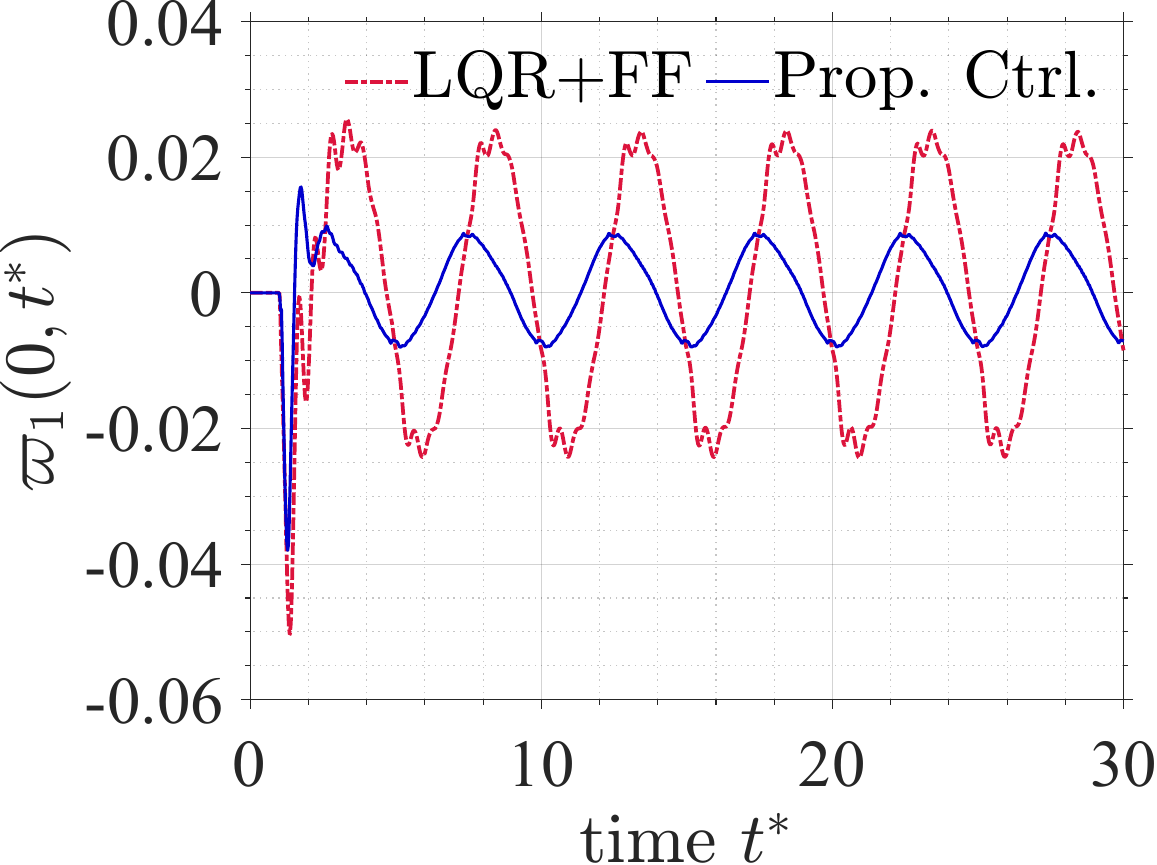}
			\end{subfigure}
			\hfill
			\begin{subfigure}[b]{0.3\textwidth}
				\centering
				\includegraphics[width=\textwidth]{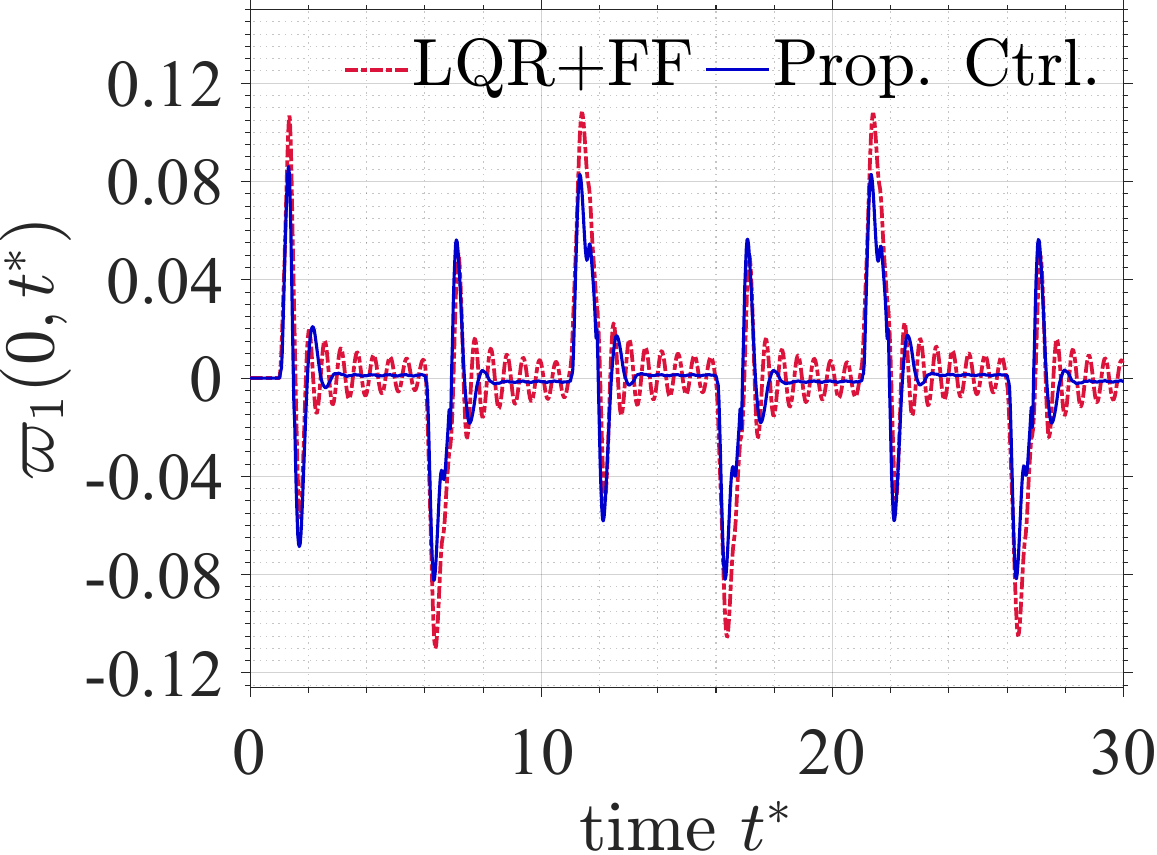}
			\end{subfigure}
			\hfill
			\begin{subfigure}[b]{0.3\textwidth}
				\centering
				\includegraphics[width=\textwidth]{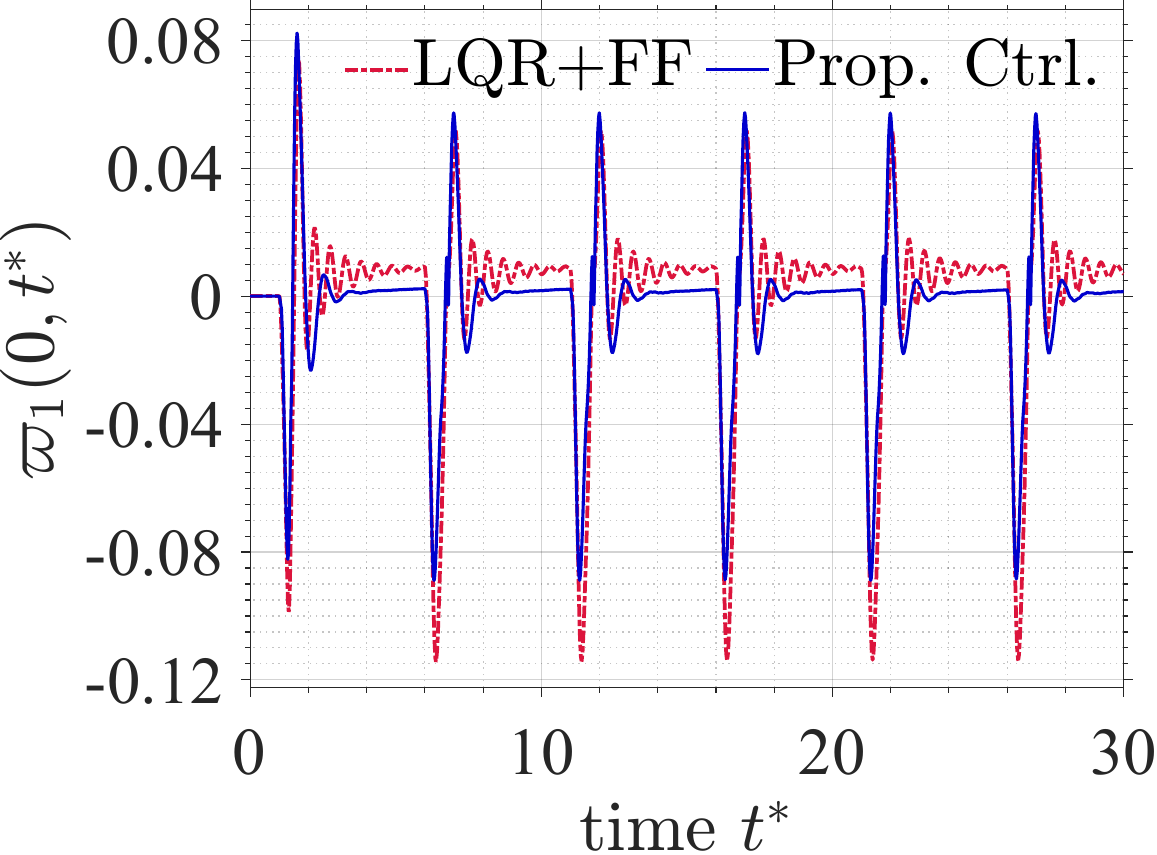}
			\end{subfigure}
			\caption{Tip tracking error $\varpi_1(1,t^\ast)$.}
			\label{fig:tip_error}
		\end{subfigure}
		
		\caption{Performance validation of the primary link-joint control subsystem under the sinusoidal (left), square (middle), and sawtooth (right) trajectories. (a) Angular tracking error $\Delta \theta_1(t)$ for the backstepping method (blue) and the LQR+feedforward (red) relative to the reference (green). (b)Transverse displacement $\upsilon_1(0,t)$. (c) Tip tracking error $\varpi_1(0,t)$. Trajectory parameters: Sinusoidal (Amplify: $\frac{40\pi}{180}$, Frequency: $0.2$Hz), Square (Amplify: $\frac{35\pi}{180}$, Frequency: $0.1$Hz), and Sawtooth (Amplify: $\frac{35\pi}{180}$, Frequency: $0.2$Hz).}
		\label{fig:control_link1_results}
	\end{figure*}
	
	\subsection{Trajectory Tracking on Primary Link of 2DSFMR}\label{experiment:link-joint1}
	In this section, trajectory tracking experiments are conducted on the primary link-joint of the 2DSFMR, with the secondary link modeled as a lumped mass at the distal end. For the experimental setup, the observer is intentionally initialized with $\hat{\xi}_1(x^\ast,0)=1$, $\hat{\eta}_1(x^\ast,0)=1$ to test its performance under large initial errors, and the control parameter is set to $\acute c_1=0.5$.
	
	To test the effectiveness of the output-feedback control framework presented in Sec. \ref{Modeling} under more challenging conditions,
	the reference trajectory $\theta_{d1}(t^\ast)$ is defined by three distinct profiles:
	
	1) Sinusoidal wave (Amplify: $40\pi/180$, Frequency: $0.2$ Hz): Assesses the high-accuracy continuous tracking performance of control strategies.
	
	2) Square wave (Amplify: $35\pi/180$, Frequency: $0.1$ Hz): Assesses the vibration suppression capabilities.
	
	3) Sawtooth wave (Amplify: $35\pi/180$, Frequency: $0.2$ Hz): Assesses both dynamic tracking precision and transient vibration suppression.
	
	We evaluate the proposed observer's performance by analyzing the slope of the transverse displacement $\varpi_{1,x^\ast}(x^\ast,t^\ast)$, reconstructed from the observer output $\hat{\eta}_1(x^\ast,t^\ast)$ and $\hat{\xi}_1(x^\ast,t^\ast)$ at spatial locations $x^\ast=0$ and $x^\ast=0.5L_1^\ast$. The corresponding results are summarized in Tab \ref{Tab:observer_performance}, which validate the capability of the observer \eqref{eq:observer_xi1}--\eqref{eq:observer_X1} to accurately estimate the states required by the controller \eqref{U_output_fedback}.
	Across all three trajectory scenarios, the results show that the state estimates of $\varpi_{1,x^\ast}(x^\ast,t^\ast)$, at $x^\ast=0$ and $x^\ast=0.5L_1^\ast$, converge rapidly, despite the large initial estimation error and the abrupt signal variations in the square and sawtooth waves. The statistical results further confirm this behavior. Although the ME reaches 0.013 due to the injected initial error, both the MAE and RMSE remain very low, demonstrating the robustness and accuracy of the proposed estimation scheme.
	
	Furthermore, the closed-loop output-feedback control performance is illustrated in Fig. \ref{fig:control_link1_results}, which provides comparative results against the baseline LQR+feedforward. The control strategies are accessed based on their the trajectory tracking and vibration suppression capabilities, quantified by the angular tracking error $\Delta \theta_1(t^\ast)$ \eqref{det_theta1}, the transverse displacement $\upsilon_1(0,t^\ast)$ \eqref{parameters_dimensionless} in the moving frame $x_{m1}y_{m1}$ of the joint representing the link vibration, the tip displacement $\varpi_1(0,t^\ast)$ \eqref{varpi} in the fixed frame $x_{f1}y_{f1}$ representing the trajectory tracking error. As shown in the Fig. \ref{fig:control_link1_results} a) and b), the proposed controller achieves a smaller joint tracking error and faster vibration suppression across all three reference trajectories compared to the LQR+feedforward approach. Fig. \ref{fig:control_link1_results} c) shows that the proposed controller yields a smaller tip tracking error than the LQR approach, and moreover, it is obvious that the proposed controller achieves faster convergence of the tip tracking error to zero under the square and sawtooth trajectories. This indicates that the proposed controller enables faster convergence of the flexible manipulator's tip to the references while achieving higher tracking accuracy than the LQR+feedforward controller.
	
	\begin{figure*}[ht]
		\centering
		\begin{subfigure}{\textwidth}
			\centering
			\begin{subfigure}[b]{0.325\textwidth}
				\centering
				\includegraphics[width=\textwidth]{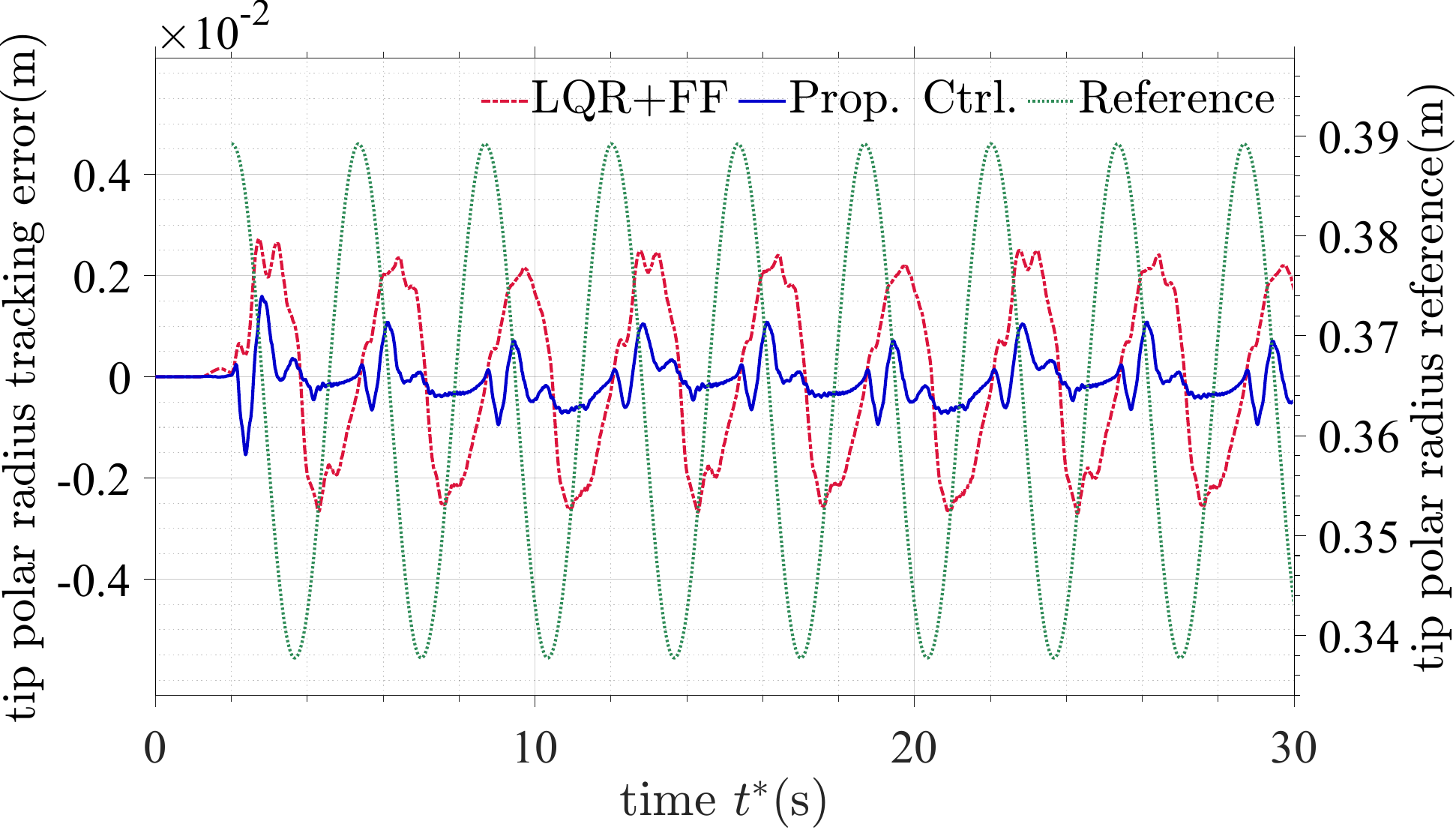}
			\end{subfigure}
			\hfill
			\begin{subfigure}[b]{0.325\textwidth}
				\centering
				\includegraphics[width=\textwidth]{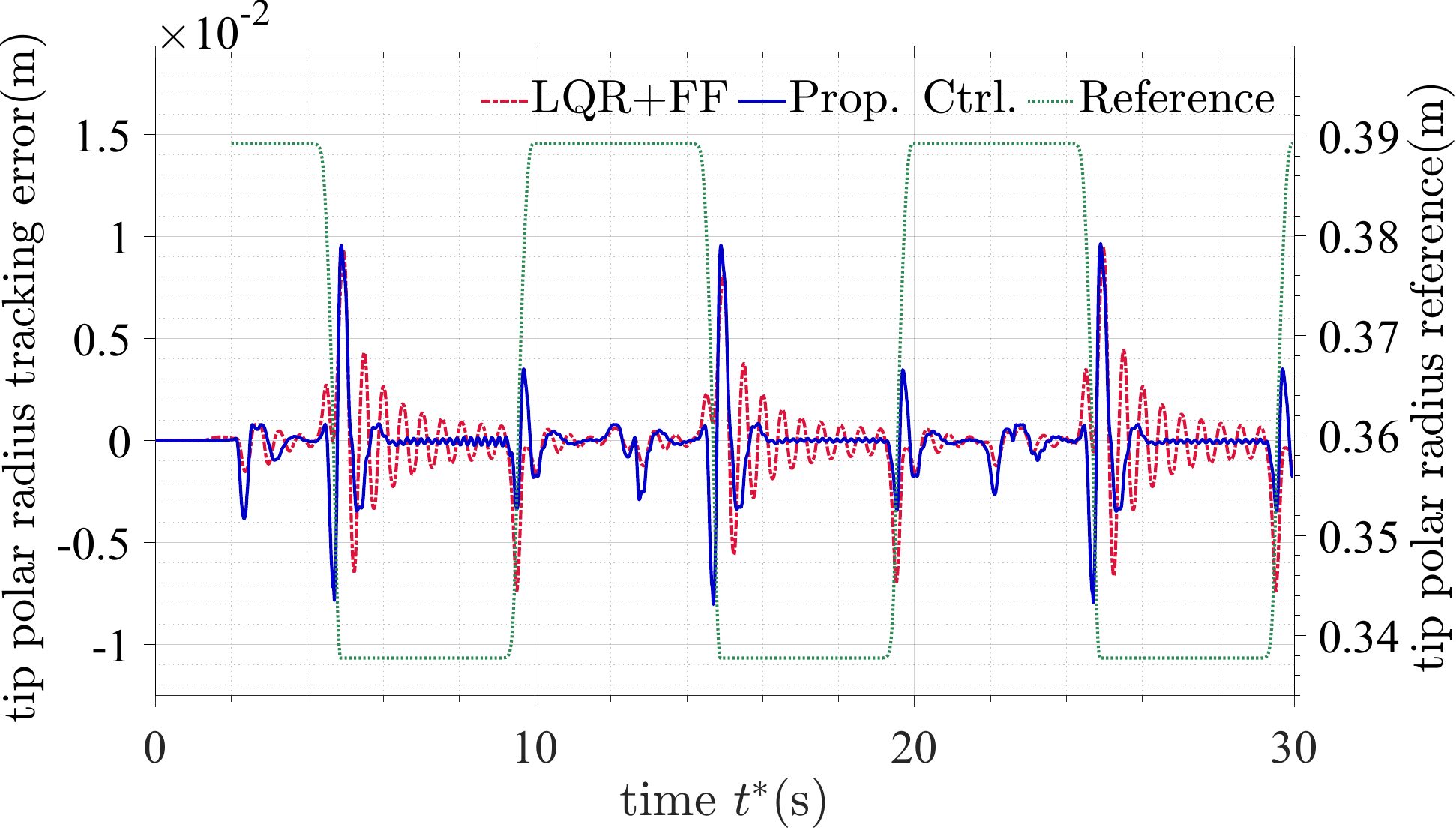}
			\end{subfigure}
			\hfill
			\begin{subfigure}[b]{0.325\textwidth}
				\centering
				\includegraphics[width=\textwidth]{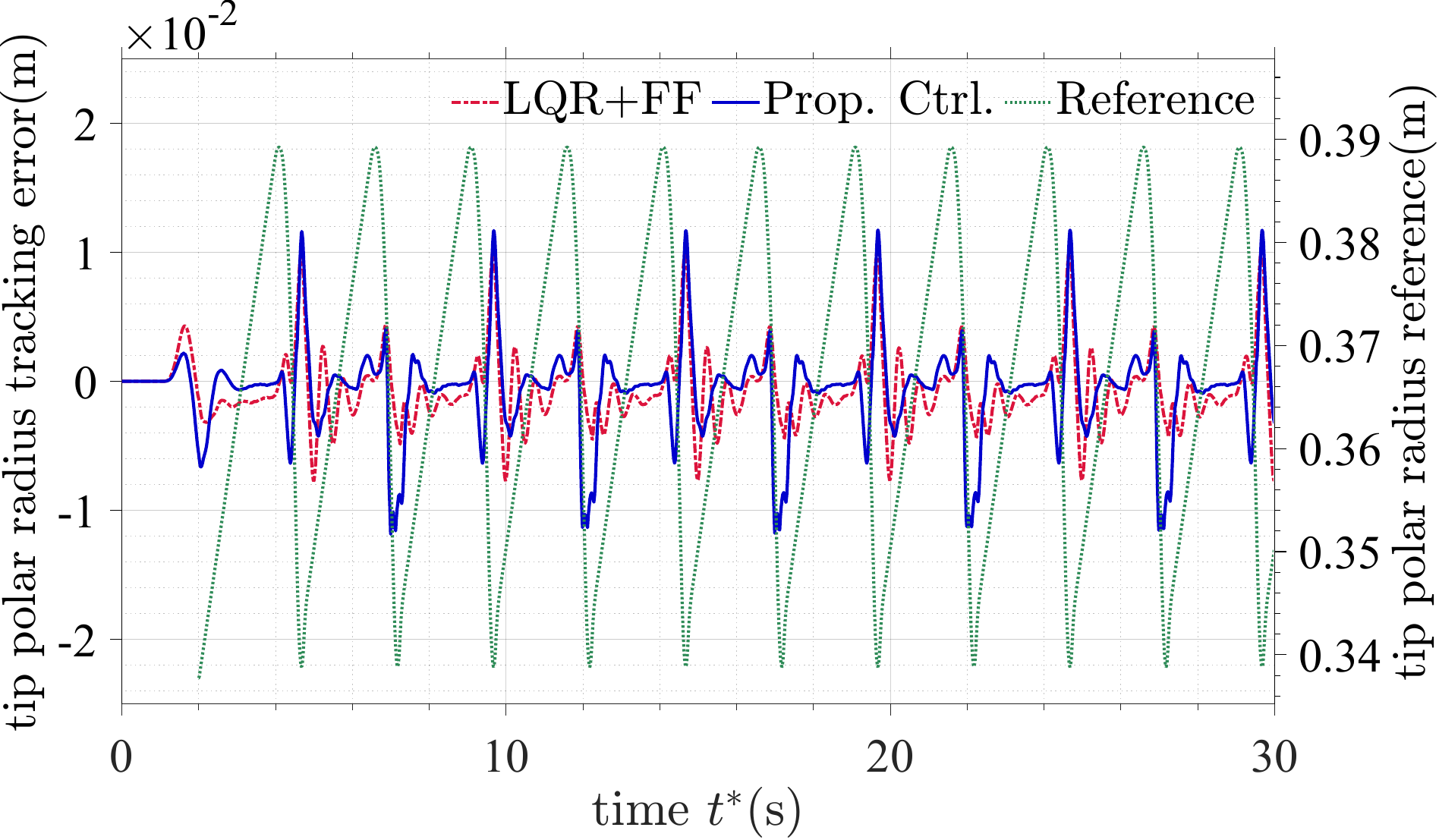}
			\end{subfigure}
			\caption{Tracking performance for tip polar radius $r(t^\ast)$.}
			\label{fig:radius_results}
		\end{subfigure}
		
		\vspace{1em} 
		
		\begin{subfigure}{\textwidth}
			\centering
			\begin{subfigure}[b]{0.325\textwidth}
				\centering
				\includegraphics[width=\textwidth]{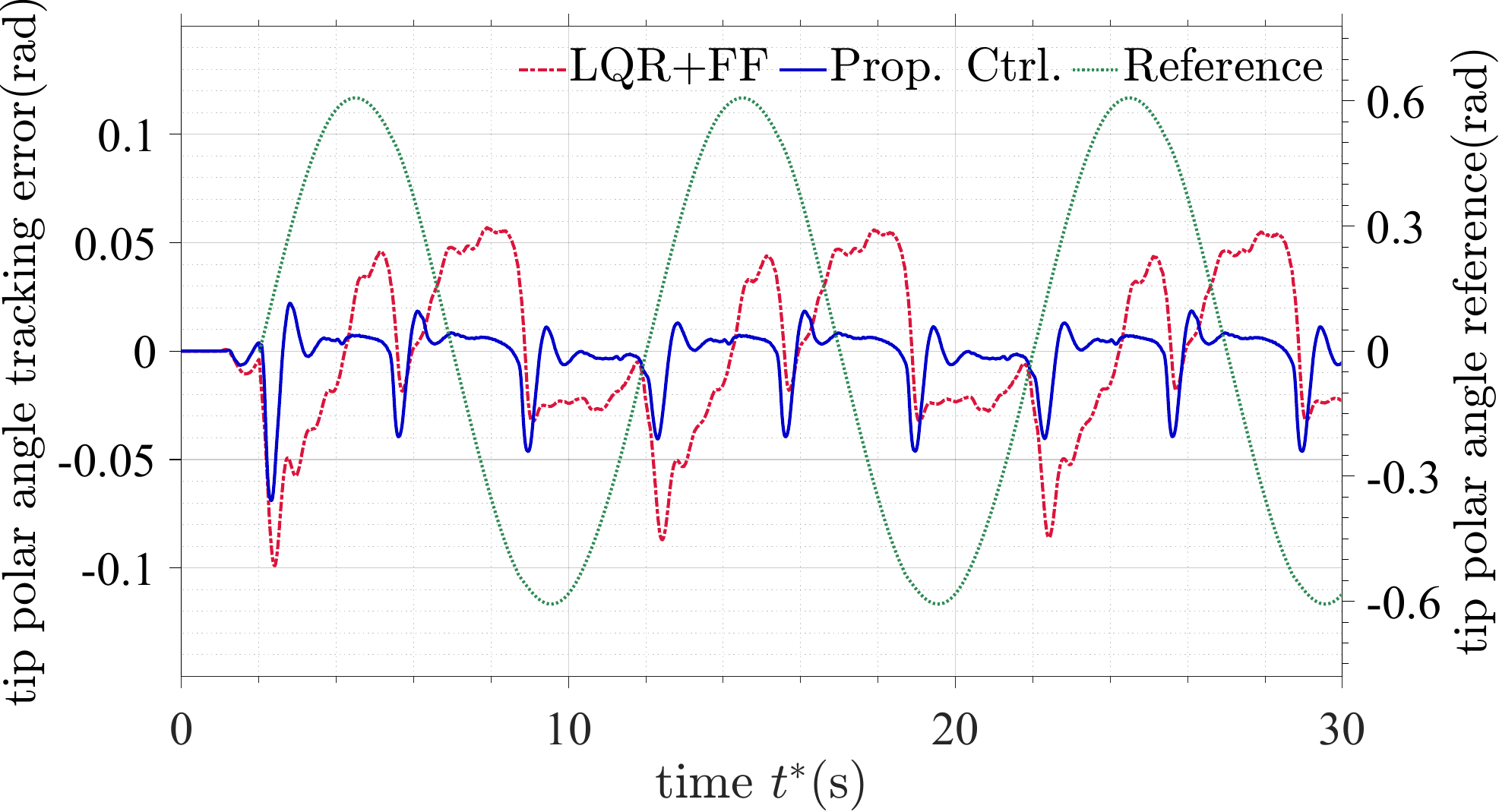}
			\end{subfigure}
			\hfill
			\begin{subfigure}[b]{0.325\textwidth}
				\centering
				\includegraphics[width=\textwidth]{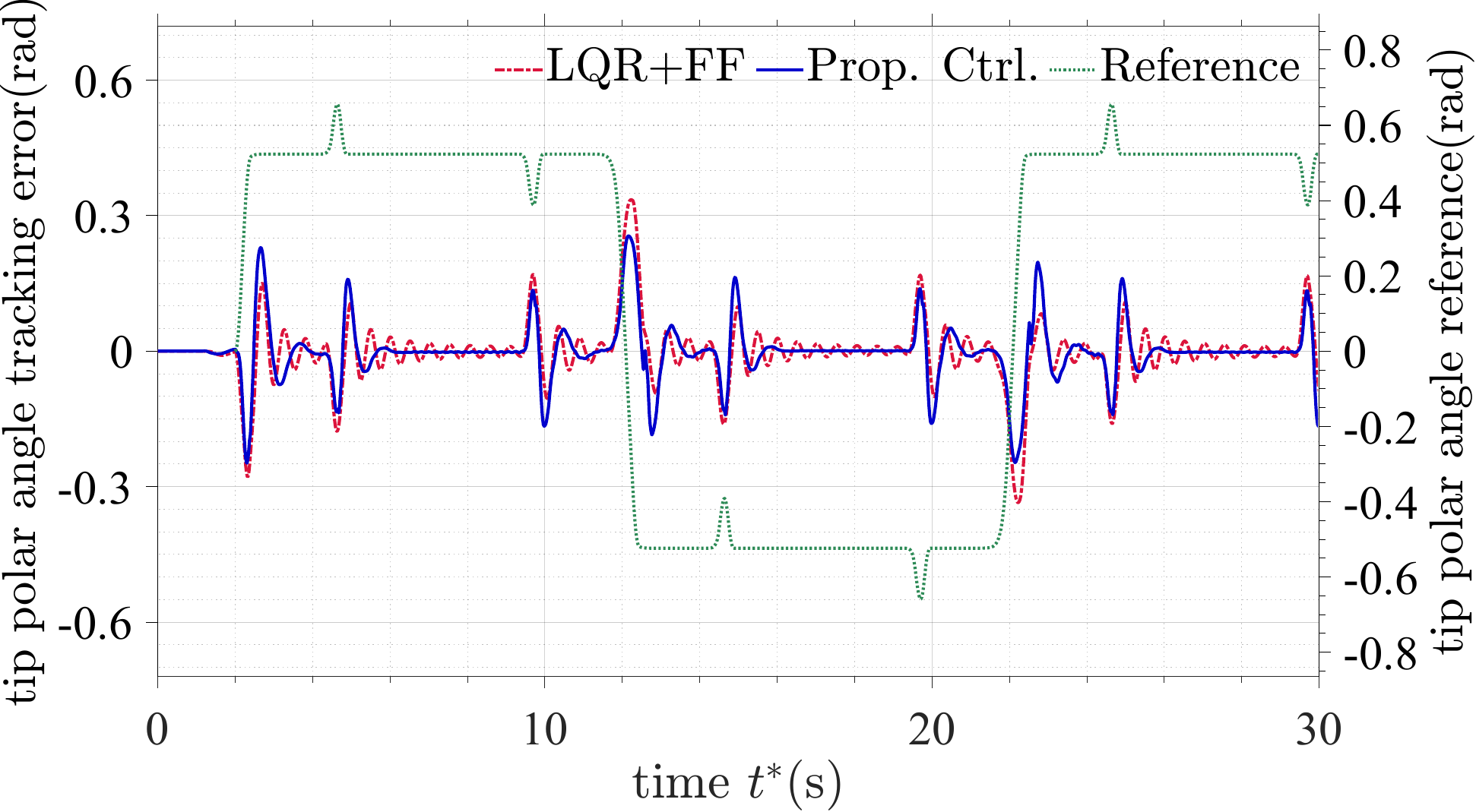}
			\end{subfigure}
			\hfill
			\begin{subfigure}[b]{0.325\textwidth}
				\centering
				\includegraphics[width=\textwidth]{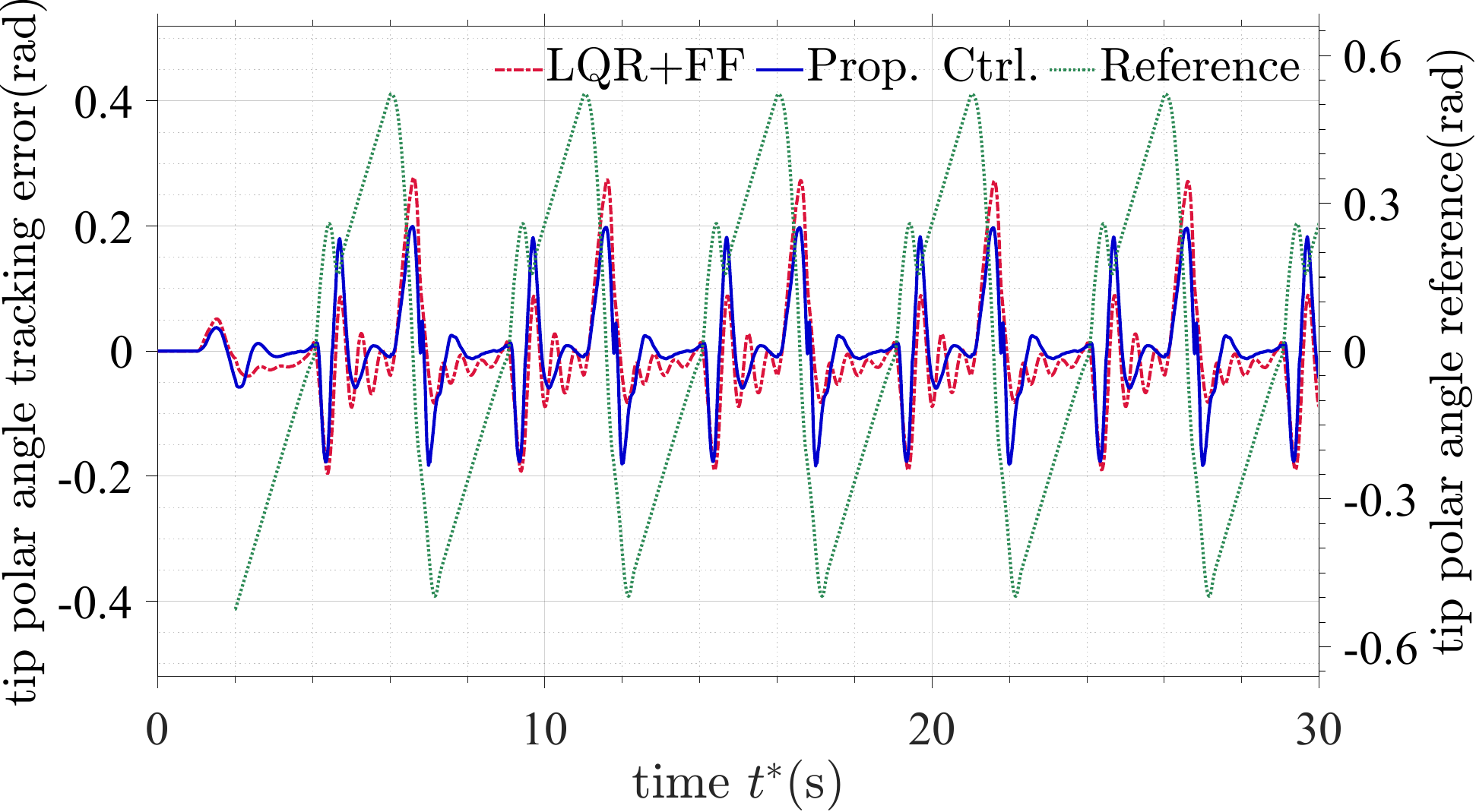}
			\end{subfigure}
			\caption{Tracking performance for tip polar angle $\varphi(t^\ast)$.}
			\label{fig:angle_results}
		\end{subfigure}
		
		\caption{Experimental performance comparison of the 2DSFMR under sinusoidal (left), square (middle), and sawtooth (right) trajectories. (a) Tip polar radius tracking error for the backstepping method (blue), the LQR+feedforward (red) relative to the reference (green). (b) Corresponding tip polar angle tracking error for the same control strategies.}
		\label{fig:complete_tracking_results}
	\end{figure*}
	
	\subsection{Trajectory Tracking on 2DSFMR}
	We implement trajectory tracking control on the 2DSFMR for three profiles defined in polar coordinates $(r,\varphi)$.
	\begin{enumerate}
		\item Sinusoidal: $r_d(t^\ast)=\acute r_1+\acute r_2\cos(1.2\pi t^\ast), \varphi_d(t^\ast)=\frac{35\pi}{180}\sin(0.4\pi t^\ast),$
		\item Square:  $r_d(t^\ast)=\acute r_1+\acute r_2sgn(\cos(0.4\pi t^\ast)), \varphi_d(t^\ast)=\frac{35\pi}{180}sgn(\sin(0.2\pi t^\ast)),$	
		\item Sawtooth: $r_d(t^\ast)=\acute r_1+\acute r_2(\lfloor 0.4t^\ast+0.5 \rfloor-0.4t^\ast), \varphi_d(t^\ast)=\frac{35\pi}{180}(0.2t^\ast-\lfloor 0.2t^\ast+0.5 \rfloor),$
	\end{enumerate}
	where $sgn(\cdot)$ is the sign function, $\lfloor \cdot \rfloor$ is the floor function, $\acute r_1=\frac{(L_1^\ast+L_2^\ast)(2+\sqrt{3})}{4}$, $\acute r_2=\frac{(L_1^\ast+L_2^\ast)(2-\sqrt{3})}{4}$.	
	Then, the desired joint angles $\theta_{d1}(t^\ast)$ and $\theta_{d2}(t^\ast)$ obtained via inverse kinematics are expressed as follows
	\begin{align}
		&\theta_{d1}=\varphi_d-\arctan\bigg( \frac{L_2^\ast\sin(\theta_{d2})}{L_1^\ast+L_2^\ast\cos(\theta_{d2})} \bigg), \label{theta_d1}  \\
		&\theta_{d2}=\arctan\bigg( \frac{r_d^2-{L_1^\ast}^2-{L_2^\ast}^2}{2L_1^\ast L_2^\ast}\bigg). \label{theta_d2}
	\end{align}
	
	The comparative experimental results for the LQR+feedforward, the backstepping method are shown in Fig. \ref{fig:complete_tracking_results}.
	Additionally, both the tip polar radial and tip polar angle tracking errors under the backstepping is smaller than those of the LQR+Feedforward method, and furthermore, both proposed controllers achieve faster convergence to zero of the end-effector tracking error of 2DSFMR under the square and sawtooth trajectories.
	\begin{Rem}
		The practical reference trajectories for the joints differ slightly from the ideal profiles $\theta_{d1}(t^\ast)$ and $\theta_{d2}(t^\ast)$ around discontinuous points, due to smoothing imposed to meet the speed and acceleration constraints of the motors.
	\end{Rem}
	
	\section{Conclusion}\label{Conclusion}
	In this article, we propose a backstepping output-feedback control framework for fast vibration suppression and tip trajectory tracking of an n-degree-of-freedom serial flexible manipulator robot (nDSFMR). We begin by analyzing the $i$-th link-joint of the nDSFMR, which is modeled as a slender Timoshenko beam coupled with an ODE derived via Hamilton's principle. Based on this dynamic model, we first develop a backstepping state-feedback controller. This is subsequently extended to an output-feedback approach by designing an observer that reconstructs the unmeasured states from only boundary measurements of the base strain $E_{bi}(t)$ and the joint angle $\theta_i(t)$.
	Finally, the experimental validation is conducted on a physical 2DSFMR platform. Comparative results against the baseline LQR+feedforward method demonstrate the superior capability of our proposed strategy in achieving fast vibration suppression and trajectory tracking.
	
	Future work will incorporate adaptive control to handle sudden payload uncertainties (e.g., during end-effector grasping) and integrate optimal motion planning to avoid high-gradient or discontinuous points in reference trajectories that challenge actuator limits.
	
	\section*{Appendix}\label{appendix}
	\setcounter{subsection}{0}
	\setcounter{section}{0}
	
	\subsection{Expression of $h_1\thicksim H_7$}\label{h}
	\setcounter{equation}{0}
	\renewcommand{\theequation}{A.\arabic{equation}}
	Expressions of  $h_1$, $h_2$, $h_3$, $h_4$, $h_5$, $H_6$, $H_7$ in \eqref{bet} are shown as follows
	\begin{align}
		h_1=&\frac{c_i}{J_i}-\frac{\rho(1,1)}{\sqrt{\epsilon_i}}, \quad
		h_2=\frac{1}{\sqrt{\epsilon_i}}\rho(1,0)+\lambda(1)B_i, \notag \\
		h_3=&\frac{1}{\sqrt{\epsilon_i}}\sigma(1,1)+\frac{c_i}{J_i}, \quad h_4=-\frac{1}{\sqrt{\epsilon_i}}\sigma(1,0), \notag \\
		h_5=&-\frac{b_i^2}{2\sqrt{\epsilon_i}}\int_{0}^{1}\int_{0}^{y}\cosh(b_i(y-z))\sigma(1,y)\lambda(z)dzdy \notag \\
		&-\frac{c_i}{J_i}\lambda(1)+\lambda(1)(A_i+B_iK_i), \notag \\
		H_6(y)=& -\frac{b_i^2}{2\sqrt{\epsilon_i} }\int_{y}^{1}\bigg(\int_{y}^{z}\cosh(b_i(z-\tau))\sigma(\tau,y)d\tau \notag \\
		&+\cosh(b_i(z-y)) \bigg)\sigma(1,z)dz+\frac{c_i}{J_i}\rho(1,y) \notag \\
		&+\frac{1}{\sqrt{\epsilon_i}}\rho_y(1,y), \notag \\
		H_7(y)=&-\frac{b_i^2}{2\sqrt{\epsilon_i} }\int_{y}^{1}\bigg( \int_{y}^{z}\cosh(b_i(z-\tau))\sigma(\tau,y)d\tau \notag \\
		&-\cosh(b_i(z-y)) \bigg)\sigma(1,z)dz-\frac{1}{\sqrt{\epsilon_i}}\sigma_y(1,y)\notag \\
		&+\frac{c_i}{J_i}\sigma(1,y).
	\end{align}
	
	\subsection{Expression of $n_1\thicksim N_7$}\label{n}
	\setcounter{equation}{0}
	\renewcommand{\theequation}{B.\arabic{equation}}
	Expressions of  $n_1$, $n_2$, $n_3$, $n_4$, $n_5$, $N_6$, $N_7$ in \eqref{U_origin} are given as follows
	\begin{align}
		n_1=&-\frac{1}{2\sqrt{\epsilon_i}R_i}(\acute{c}_i+h_1),\quad n_2=-\frac{1}{2\sqrt{\epsilon_i}R_i}h_3, \notag \\
		n_3=&-\frac{1}{2\sqrt{\epsilon_i}R_i}\bigg((\acute{c}_i+h_1)\gamma(1)+h_2\gamma(0)+h_4C_i+h_5\notag \\
		&+\int_{0}^{1}H_7(y)\gamma(y)dy\bigg), \notag \\
		n_4=&-\frac{1}{2\epsilon_iR_i}(h_2-h_4),\quad n_5=-\frac{1}{2\epsilon_iR_i}, \notag \\
		N_6(y)=&\frac{1}{2\sqrt{\epsilon_i}R_i}\bigg((\acute{c}_i+h_1)k(1,y)+\frac{b_i^2}{2\sqrt{\epsilon_i}}\cosh(b_i(1-y)) \notag \\
		&-H_7(y)+\int_{y}^{1}k(z,y)H_7(z)dz\bigg), \notag \\
		N_7(y)=&\frac{1}{2\sqrt{\epsilon_i}R_i}\bigg((\acute{c}_i+h_1)l(1,y)-\frac{b_i^2}{2\sqrt{\epsilon_i}}\cosh(b_i(1-y)) \notag \\
		&-H_8(y)+\int_{y}^{1}l(z,y)H_7(z)dz\bigg). \label{eq:n10}
	\end{align}
	
	\subsection{Proof of Theorem \ref{Thm:state-feedback}}\label{Proof:Theorem1}
	\setcounter{equation}{0}
	\renewcommand{\theequation}{C.\arabic{equation}}
	\textit{1)} We define a Lyapunov function $V_1(t)$ as
	\begin{align}
		V_1(t)=&\frac{\sqrt{\epsilon_i}}{2}\int_{0}^{1}e^{-\acute ax}\eta_i^2(x,t)dx+\frac{\sqrt{\epsilon_i}}{2}\int_{0}^{1}\acute he^{x}\beta_i^2(x,t)dx \notag \\
		&+\frac{1}{2}\beta_i(1,t)^2+X_i(t)^TP_1X_i(t) \label{V1}
	\end{align}
	where the positive parameters $\acute h$, $\acute a$ are to be chosen later. Recalling $A_i+B_iK_i$ is Hurwitz, there exists a matrix $P_1=P_1^T>0$ being the solution to the Lyapunov equation
	$P_1(A_i+B_iK_i)+(A_i+B_iK_i)^TP_1=-Q_1$
	for some $Q_1=Q_1^T>0$.
	
	Defining
	$$
	\Omega_1(t)=| \beta_i(1,t)|^2+| X_i(t) |^2+\| \beta_i(\cdot,t) \|^2+\| \eta_i(\cdot,t) \|^2,
	$$
	we also have
	$
	\theta_{1,1}\Omega_1(t)\leq V_1(t)\leq \theta_{1,2}\Omega_1(t)
	$
	where
	$
	\theta_{1,1}=\min\{ $ $ \frac{1}{2},\lambda_{\min}(P_1),\frac{\sqrt{\epsilon_i}\acute h}{2},\frac{\sqrt{\epsilon_i}e^{-\acute a}}{2} \} ,
	\theta_{1,2}=\max\{ \frac{1}{2},$ $\lambda_{\max}(P_1),\frac{\sqrt{\epsilon_i}\acute he}{2},\frac{\sqrt{\epsilon_i}}{2} \}.
	$
	Taking the time derivative of $V_1(t)$ along \eqref{eq:beta1}--\eqref{eq:X1_target}, \eqref{eq:bc_beta1},
	applying integral by parts, using Cauchy-Schwartz inequality, we get
	\begin{align}
		\dot{V}_1(t) &\leq -\frac{\lambda_{\min}(Q_1)}{2}| X_i(t) |^2+\frac{4| P_1B_i |^2}{\lambda_{\min}(Q_1)}\beta_i(0,t)^2 \notag \\
		&+\frac{4| P_1D_i |^2}{\lambda_{\min}(Q_1)}\ddot{\theta}_{di}(t)^2+\int_{0}^{1}\acute he^x\beta_i(x,t)\sqrt{\epsilon_i}\beta_{i,t}(x,t)dt
		\notag \\
		&+\int_{0}^{1}e^{-\acute ax}\eta_i(x,t)\sqrt{\epsilon_i}
		\eta_{i,t}(x,t)dt-\acute c_i \beta_i(1,t)^2 \notag \\
		&\leq-\bigg(\frac{\lambda_{\min}(Q_1)-\bar L_1}{2}-|C_i-K_i|^2 \bigg)| X_i(t) |^2 \notag \\
		&-\bigg(\frac{\acute h}{2}-\frac{4| P_1B_i |^2}{\lambda_{\min}(Q_1)}-1 \bigg)\beta_i(0,t)^2-(\acute{c}_i-\frac{\acute he}{2})\beta_i(1,t)^2 \notag \\
		&-\bigg(\frac{\acute a}{2}-\frac{\bar L_1}{2\acute a}-2\bar L_1 \bigg)\int_{0}^{1}e^{-\acute ax}\eta_i^2(x,t)dx \notag \\
		&-\bigg(\frac{\acute h}{2}(1-\frac{\acute G}{\acute \gamma})-\frac{\bar L_1}{2\acute a}\bigg)\int_{0}^{1}e^{x}\beta_i^2(x,t)dx\notag \\
		&+\bigg( \frac{\bar L_1+\acute he\acute G\acute \gamma}{2}+\frac{4| P_1D_i |^2}{\lambda_{\min}(Q_1)} \bigg)\ddot{\theta}_{di}(t)^2 \label{dot_V1}
	\end{align}
	where
	$\bar L_1>\max\{ |\epsilon_i(1+R_i-x)|_\infty, |\frac{b_i^2}{2}\int_{0}^{x}\cosh(b_i(x-y))\lambda(y)dy|_\infty, \frac{b_i^2}{2}\| -\cosh(b_i(x-y))+\int_{y}^{x}\cosh(b_i(x-z))\sigma(z,y)dz \|_\infty, \frac{b_i^2}{2}\| \cosh(b_i(x-y))+\int_{y}^{x}\cosh(b_i(x-z))\rho(z,y)dz \|_\infty \}$, $\acute G=|G(x)|_\infty$, and $\acute \gamma$ is a positive parameter to be chosen later.
	
	Choosing $\acute h,\acute{c}_i$, $\acute a$ and $\acute \gamma$  to satisfy
	\begin{align}
		\acute h&> \max\{ \frac{8| P_1B_i|^2}{\lambda_{\min}(Q_1)}+2,\frac{\bar L_1\acute \gamma}{(\acute \gamma-\acute G)\acute a}\},\quad \acute{c}_i>\frac{\acute he}{2}, \notag \\
		\acute a&>4\bar L_1+\sqrt{16\bar L_1^2+4\bar L_1},\quad \acute \gamma>\acute G \label{condition_V1}
	\end{align}
	we obtain
	$$
	\dot{V_1}(t)\leq  -\acute \lambda_1\theta_{1,2}\Omega_1(t)+D_\theta
	$$
	for sufficiently large $\lambda_{\min}(Q_1)$ and where
	\begin{align}
		&\acute \lambda_1=\min\bigg\{ \frac{\lambda_{\min}(Q_1)-\bar L_1}{2}-|C_i-K_i|^2, \acute{c}_i-\frac{\acute he}{2}, \notag \\
		&  \frac{\acute h}{2}-\frac{4| P_1B_i |^2}{\lambda_{\min}(Q_1)}-1,  \frac{\acute a}{2}-\frac{\bar L_1}{2\acute a}-2\bar L_1, \frac{\acute h}{2}(1-\frac{\acute G}{\acute \gamma})-\frac{\bar L_1}{2\acute a}, \notag \\
		& \frac{\bar L_1+\acute he\acute G\acute \gamma}{2}+\frac{4| P_1D_i |^2}{\lambda_{\min}(Q_1)}  \bigg\}, \label{eq:acute_lambda_1}
	\end{align}
	and $D_\theta=( \frac{\bar L_1+\acute he\acute G\acute \gamma}{2}+\frac{4| P_1D_i |^2}{\lambda_{\min}(Q_1)} )\ddot{\theta}_{di}(t)^2$ is bounded according to \eqref{eq:conditons_theta_id}.	
	
	Then from the comparison principle, we can conclude $V_1(t)< V_1(0)e^{-\acute \lambda_1 t}+\frac{D_\theta}{\acute \lambda_1}$ which implies that $\Omega_1(t)< \frac{\theta_{1,2}}{\theta_{1,1}}\Omega_1(0)e^{-\acute \lambda_1 t}+\frac{D_\theta}{\theta_{1,1}\acute \lambda_1}$ with $\lambda_{\min}(Q_1)$ influenced by $K_i$ since $P_1(A_i+B_iK_i)+(A_i+B_iK_i)^TP_1=-Q_1$.
	
	Moreover, the analysis on the higher-order terms $\|\eta_i(x,t)\|_{H^1}$, $\|\xi_i(x,t)\|_{H^1}$, which will also be required in analyzing the boundedness of the control input, is necessary. The explicit deduction has been shown in \cite{bib0}, so we only give a short proof for simplicity.
	
	Differentiating \eqref{eq:beta1}, \eqref{eq:eta1_target} with respect to $x$, differentiating \eqref{eq:bc_beta1}, \eqref{eq:bc_eta1_target} with respect to $t$, we get the equations of $\beta_{i,x}(x,t)$, $\eta_{i,x}(x,t)$. Based on these, it straightforward to obtain their explicit solutions.
	
	Defining
	$$
	V_2(t)=V_1(t)+\frac{1}{3}\int_{0}^{1}\beta_{i,x}^2(x,t)dx+\frac{1}{12}\int_{0}^{1}\eta_{i,x}^2(x,t)dx, \label{V2}
	$$
	and considering
	$$
	\Omega_2(t)=|\beta_i(1,t)\lvert^2+| X_i(t) |^2+\| \beta_i(\cdot,t) \|^2_{H^1}+\| \eta_i(\cdot,t) \|^2_{H^1},
	$$
	we have $\theta_{2,1}\Omega_2(t)\leq V_2(t)\leq \theta_{2,2}\Omega_2(t)$ for some positive $\theta_{2,1}$, $\theta_{2,2}$. Recalling \eqref{dot_V1}, using Cauchy Schwarz inequality and Young's inequality, we obtain
	\begin{align}
		&V_2(t)\leq \frac{\kappa_1}{\theta_{1,1}}\theta_{1,1}\Omega_1(t)+V_1(0)e^{-\acute \lambda_1 t}+\frac{D_\theta}{\acute \lambda_1}+(\acute G^2+ \notag \\
		&\epsilon_i|(C_i-K_i)D_i|^2+2\epsilon_i^2(1+R_i)^2+|G^{(1)}(x)|_\infty^2)\ddot{\theta}_{di}(t)^2
		\label{V2_v2}
	\end{align}
	for some positive $\kappa_1$. We thus have $V_2(t)\leq (1+\frac{\kappa_1}{\theta_{2,1}})$ $V_2(0)e^{-\acute \lambda_1t}+\mathcal{C}_1$ where
	\begin{align}\
		\mathcal C_1=&(\acute G^2+\epsilon_i|(C_i-K_i)D_i|^2+|G^{(1)}(x)|_\infty^2+2\epsilon_i^2(1+R_i)^2) \notag \\
		&\ddot{\theta}_{di}(t)^2+(1+\frac{\kappa_1}{\theta_{2,1}})\frac{D_\theta}{\acute \lambda_1}. \label{eq:math_C1}
	\end{align}
	It means that $\Omega_2(t)\leq \frac{\theta_{2,2}}{\theta_{2,1}}(1+\frac{\kappa_1}{\theta_{1,1}})\Omega_2(0)e^{-\acute \lambda_1 t}+\frac{\mathcal{C}_1}{ \theta_{2,1}}$.
	Additionally, recalling \eqref{eq:bc_alpha}, \eqref{varpi}, \eqref{det_theta1}, and \eqref{xi1}--\eqref{acute x1}, $\Delta \theta_i$ is expressed as
	$$
	\Delta \theta_i(t)=\frac{1}{2R_i}\int_{0}^{1}(\xi_i(x,t)-\eta_i(x,t))dx+\frac{\acute x_i}{R_i},
	$$
	and it is straightforward to get
	\begin{align}
		|\Delta \theta_i(t)|^2\leq {\kappa}_3(\|\eta_i(\cdot,t)\|^2+|X_i(t)|^2+\|\xi_i(\cdot,t)\|^2) \label{eq:deta_theta1}
	\end{align}
	for some positive $\kappa_3$.
	Thus, we get $\Delta \theta_i(t)$ is bounded. Next, defining
	\begin{align}
		\Xi_1(t)=&| \Delta \theta_i(t)|^2+| \Delta \dot{\theta}_i(t)|^2+| X_i(t)|^2+\| \xi_i(\cdot,t)\|^2_{H^1} \notag \\
		&+\| \eta_i(\cdot,t)\|^2_{H^1}, \label{Xi1}
	\end{align}
	applying Cauchy-Schwarz inequality, recalling transformations \eqref{Trans:beta1}, \eqref{Trans:xi1}, and \eqref{V2_v2}, \eqref{eq:deta_theta1} we get
	$
	\theta_{2,3}\Xi_1(t)\leq \Omega_2(t)\leq \theta_{2,4}\Xi_1(t)
	$
	for some positive $\theta_{2,3}$ and $\theta_{2,4}$. Therefore we have
	\begin{align}
		\Xi_1(t)\leq& \frac{\theta_{2,2}\theta_{2,4}}{\theta_{2,1}\theta_{2,3}}(1+\frac{\kappa_1}{\theta_{1,1}})\Xi_1(0)e^{-\acute \lambda_1 t}+\frac{\mathcal{C}_1}{\theta_{2,1}\theta_{2,3}}.
	\end{align}
	Thus, recalling \eqref{V2_v2}, \eqref{V1},  \eqref{Omeg0} is achieved with
	\begin{align}
		\Upsilon_0=&\frac{\theta_{2,2}\theta_{2,4}}{\theta_{2,1}\theta_{2,3}}(1+\frac{\kappa_1}{\theta_{1,1}}),\quad \acute{C}_0=\acute \lambda_1,\quad
		\mathcal{C}_0=\frac{\mathcal{C}_1}{\theta_{2,1}\theta_{2,3}}.  \label{mathC0}
	\end{align}
	
	\textit{ 2):}
	Applying Cauchy Schwarz inequality into \eqref{U_origin}, recalling \eqref{eq:n10}, together with Properties 1, Property 2 is obtained.
	
	The proof of Theorem \ref{Thm:state-feedback} is complete.

\end{document}